\documentclass{article}

\PassOptionsToPackage{numbers, compress}{natbib}

\usepackage[preprint]{neurips_2026} 

\usepackage[utf8]{inputenc}
\usepackage[T1]{fontenc}
\usepackage{hyperref}
\usepackage{url}
\usepackage{booktabs}
\usepackage{amsfonts}
\usepackage{amsmath}
\usepackage{amsthm}
\usepackage{nicefrac}
\usepackage{microtype}
\usepackage{xcolor}
\usepackage[textsize=tiny]{todonotes}

\usepackage{amssymb}
\usepackage{mathtools}

\usepackage{cleveref}
\usepackage{algorithm}
\usepackage{algorithmic}
\usepackage{dsfont}

\usepackage{multirow}
\usepackage{graphicx}
\usepackage{makecell}
\usepackage{wrapfig}

\theoremstyle{plain}
\newtheorem{theorem}{Theorem}[section]

\theoremstyle{definition}

\theoremstyle{remark}

\newcommand{\std}[1]{\textsuperscript{\textcolor{gray}{\tiny$\pm$#1}}}

\newcommand{\squeezeup}{\vspace{-2.5mm}}

\title{Filter-then-Weight: Online Data Selection and Reweighting for LLM Fine-Tuning}

\author{%
  Fangxin Wang, Peyman Baghershahi, Langzhou He,\\
  \textbf{Henry Peng Zou, Sourav Medya, Philip S. Yu}\\
  Department of Computer Science\\
  University of Illinois Chicago\\
  \texttt{\{fwang51,pbaghe2,lhe24,pzou3,medya,psyu\}@uic.edu}
}

\begin{document}

\maketitle

% ============================
% Abstract
% ============================

\begin{abstract}
Gradient-based data selection offers a principled framework for estimating sample utility in large language model (LLM) fine-tuning, but existing methods are mostly designed for offline settings. They are therefore less suited to online fine-tuning, where data arrives sequentially, sample utility is step-dependent, and the effective update geometry is shaped by adaptive optimizers. We propose an optimizer-aware framework for gradient-based online data selection and reweighting in LLM fine-tuning. Our key idea is to view online selection not as static sample ranking, but as shaping the next target-oriented update under the current optimizer state. We formulate this as an optimizer-aware update-matching problem, establish its connection to second-order target utility, and show why subset-level construction must account for interactions and redundancy among selected samples. Based on this view, we develop a two-stage Filter-then-Weight algorithm that first filters geometrically useful candidates and then optimizes their coefficients. To make the framework practical for LLMs, we introduce a factorized outer-product gradient representation and optimized matrix computations for long-context data. Experiments show that our method consistently improves convergence and downstream performance over existing online data selection baselines under the same data budget.
\end{abstract}

% ============================
% Main Sections
% ============================

\section{Introduction}
\label{sec: intro}

Data selection for large language models (LLMs) aims to curate a representative subset from massive training corpora, enabling models to achieve convergence with reduced data consumption and improved performance. While heuristic metrics—such as representation similarity~\citep{zhang2018unreasonable,hanawaevaluation,xie2023data} and sample learning difficulty~\citep{li2024quantity}—offer simple filtering criteria, gradient-based methods~\citep{TracIn,xia2024less} provide a theoretically grounded framework that explicitly quantifies a sample's contribution to model parameter updates~\citep{koh2017understanding, grosse2023studying}. However, these methods generally require additional forward and backward passes to compute sample gradients, introducing severe computational overhead that scales linearly with the number of training samples and model parameters.

% related work data selection
Current gradient-based data selection methods can be broadly categorized into two types~\citep{qinunleashing}: gradient-based influence and gradient matching. Influence methods focus on selecting the most impactful individual data points by evaluating the similarity between a training sample's gradient and the validation gradient. Early works~\citep{koh2017understanding, TracIn} estimate this via the inner product of first-order gradient approximations. Recent adaptations for instruction tuning, such as~\cite{xia2024less}, extend to accommodate the Adam optimizer and variable instruction lengths, while~\cite{wang2024greats} introduces an online algorithm that dynamically corrects sample influence based on previously selected points. Conversely, gradient matching methods~\citep{killamsetty2021grad} optimize the weights of training samples to align the selected batch with a target distribution. For instance, at the domain level, \cite{fan2024doge} updates domain weights to maximize the inner product between the weighted training gradients and the target domain gradient. Similarly, \cite{zhang2025tagcos} applies greedy coreset selection to clustered gradients for task-agnostic coverage.
% INFLUENTIAL LANGUAGE DATA SELECTION VIA GRADIENT TRAJECTORY PURSUIT unpublished, not count

%maybe add which optimizer they usually use and why Adam is non-trivial

% difference between online and offline data selection
Most aforementioned approaches~\citep{TracIn,xia2024less} are designed for the offline setting, where the full dataset is available for pre-computing static gradients. Whereas in online fine-tuning~\citep{aljundi2019gradient}, the training data arrives sequentially (e.g., continual instruction tuning or incremental memory updates), and selection decisions must be made on-the-fly without full corpus access. In this dynamic scenario, sample's utility is inherently dependent on the model’s current parameters, creating a sequential, step-wise decision problem. This differs fundamentally from offline methods, which rely on static gradient estimates and treat selection as a single-step problem, ignoring the impact of training order. 
Furthermore, the standard offline strategy of averaging gradients over multiple historical checkpoints becomes computationally infeasible, as strict latency and storage constraints preclude the maintenance of extensive model history in a streaming setting.
Beyond these structural and efficiency challenges, effective online selection requires strictly aligning with the specific update rules of the optimizer to ensure convergence, whereas dominant existing works~\citep{TracIn,wang2024greats,fan2024doge} typically assume simple Stochastic Gradient Descent (SGD) approximations. A very recent concurrent work~\citep{wang2026opus} also targets 
online optimizer-aware data selection, but in the pre-training 
setting and within the individual-scoring paradigm; it therefore 
does not address interactions among jointly selected samples that 
arise in subset-level reweighting.

These observations suggest that online data selection should be viewed not as static sample ranking, but as shaping the next training update toward the target task—jointly accounting for both optimizer geometry and sample interactions. Motivated by this view, we propose an optimizer-aware framework for online data selection and reweighting in LLM fine-tuning. Our key idea is to formulate the problem as \emph{subset-level gradient matching}: rather than scoring samples individually~\citep{TracIn,xia2024less,wang2026opus}, we optimize a non-negative weighting over a candidate subset so that its composite gradient approximates the target-oriented update under the geometry induced by the optimizer state.

Building on this formulation, we develop a two-stage \emph{Filter-then-Weight} algorithm that decouples candidate filtering from coefficient optimization: a greedy residual filter identifies geometrically useful samples under unit weights, and a constrained reweighting stage refines their coefficients to form a precise composite update. To make this framework tractable for LLMs, we use a factorized outer-product gradient representation that preserves more information than full-gradient compression under the same memory budget, and further optimize matrix multiplications for long-context data. In addition, our formulation reveals a connection between gradient matching and second-order target utility, which provides a principled interpretation for subset-level update construction and redundancy control. Experimental results demonstrate that our method outperforms current online data selection algorithms under the same data budget while improving downstream performance. Our ablations further show that optimizer-awareness is essential for reweighting, and decoupling filtering from coefficient optimization leads to more robust and stable online selection.
\squeezeup
\section{Problem Setup}
\label{sec: prelim}
\squeezeup
\subsection{Online Training Data Selection}
\label{setting: dataselection}
We consider a large-scale training corpus, denoted as $\mathcal{D}_{tr} := \{x_i\}^{N}_{i=1}$, which has mixed knowledge type and data quality. Rather than minimizing the training loss over $\mathcal{D}_{tr}$, our goal is to optimize the model performance on the \emph{target distribution} from a different distribution. Let $\mathcal{D}_{tar}$ denote a small validation or target dataset drawn from the downstream task distribution, and let $\mathcal{L}_{\mathrm{tar}}(\theta):= \mathbb{E}_{x_j \sim \mathcal{D}_{tar}} [l(\theta, x_j)]$, where $l$ is the loss function.
Note that the validation set typically is considerably smaller than the training set, i.e., $|\mathcal{D}_{tar}| \ll |\mathcal{D}_{tr}|$. Given limited validation data, our primary objective is to improve the predictive performance of a fine-tuned model on a specific target task by strategically leveraging the vast but noisy training corpus. 

% For our fine-tuning setup, a naive approach of utilizing the entire training corpus is not only computationally inefficient but can also be detrimental to model performance. Such a strategy may lead to adverse phenomena like model misalignment or catastrophic forgetting. Considering standard batch-based training paradigm, an effective approach may be to select a smaller but high-quality subset strategically from $\mathcal{D}_{tr}$ that is maximally beneficial for the target task. The aim of this approach is to accelerate model convergence and mitigate the performance degradation that can arise from error accumulation, particularly during the initial stages of training.

%For our fine-tuning setup, a naive approach of utilizing the entire training corpus is not only computationally inefficient but can also be detrimental to model performance. Such a strategy may lead to adverse phenomena like model misalignment or catastrophic forgetting, where task-specific knowledge acquired during pre-training is supplanted by noisy or task-irrelevant information from the training data. Considering standard batch-based training paradigm, an effective approach may be to select a smaller but high-quality subset strategically and iteratively from $\mathcal{D}_{tr}$ that is maximally beneficial for the target task. The aim of this approach is to accelerate model convergence and mitigate the performance degradation that can arise from error accumulation, particularly during the initial stages of training.

\textbf{Online and Memory-Constrained Data Selection.}
We consider an online fine-tuning setting where training proceeds 
in a sequential batch-based manner. At each iteration $t$, the 
algorithm is presented with a candidate pool 
$\mathcal{S}^{(t)} \subset \mathcal{D}_{\mathrm{tr}}$ and must decide 
how to leverage its samples for parameter updates without persistent 
access to the full corpus. A limited memory buffer of previously 
seen samples may be revisited under a fixed memory budget, but full 
data replay is prohibited—reflecting practical constraints in 
large-scale streaming fine-tuning where selection decisions must be 
made on-the-fly based on the current model state.

% \textbf{Weighted Training Update.} We use continuous weights $w_i \in \mathbb{R}$ to select data points. When $w_i \in \{0,1\}$, our formulation reduces to standard data selection, where a subset of samples is chosen for training. In contrast, using continuous weights  enables a more expressive mechanism that subsumes both data selection and data reweighting, and allows the contribution of each sample to be adaptively modulated based on its relevance to the target task. More specifically, 
% given a candidate mini-batch $\mathcal{S}^{(t)}$, we assign a weight $w_i \in \mathbb{R}$ to each sample $x_i \in \mathcal{S}^{(t)}$. The resulting weighted training gradient at step $t$ is 
% \begin{equation}
% \label{eq: weighted moment}
% g_t(w) = \nabla \mathcal{L}_{tr}^{(t)} (\theta_{t-1}; w) = \sum_{x_i \in \mathcal{S}^{(t)}} w_i \nabla l(\theta_{t-1}, x_i),
% \end{equation}
% where $\theta_{t-1}$ denotes the model parameters before the $t$-th update, and weights $w = \{w_i\}$. The model parameters are then updated using a gradient-based optimizer to a new model $\theta_{t}$:
% \begin{equation}
% \label{eq: para update}
% \theta_t=\theta_{t-1} - \eta_t \, \mathcal{P}_t( g_t(w) ),
% \end{equation}
% where $\eta_t$ is the learning rate, $\mathcal{P}_t$ denotes the possibly adaptive function induced by the optimizer, discussed later.

\textbf{Weighted training update.}
We use continuous weights $w_i \in \mathbb{R}_{\geq 0}$ to score 
candidate samples. Setting $w_i \in \{0, 1\}$ recovers standard subset 
selection, while continuous weights subsume both selection and 
reweighting, allowing each sample's contribution to be adaptively 
modulated based on its relevance to the target task. Given a 
candidate mini-batch $\mathcal{S}^{(t)}$, the resulting weighted 
training gradient at step $t$ is
\begin{equation}
  g_t(w) = \nabla \mathcal{L}_{\mathrm{tr}}^{(t)}(\theta_{t-1}; w) 
         = \sum_{x_i \in \mathcal{S}^{(t)}} w_i \nabla \ell(\theta_{t-1}, x_i),
\end{equation}
and the model parameters are updated via
\begin{equation}
  \theta_t = \theta_{t-1} - \eta_t\, P_t(g_t(w)),
\end{equation}
where $\eta_t$ is the learning rate and $P_t$ is the possibly adaptive 
operator induced by the optimizer.

\subsection{Step-Wise Target Loss Reduction}
\label{setting: loss}
At each step $t$, our objective is to choose $w$ to directly minimize 
the target loss $\mathcal{L}_{\mathrm{tar}}(\theta_t)$ under the 
actual optimization dynamics used for fine-tuning. For brevity, we 
write $\nabla \ell_i := \nabla \ell(\theta_{t-1}, x_i)$ and 
$\nabla \ell_j := \nabla \ell(\theta_{t-1}, x_j)$ for training and 
validation gradients respectively.

To connect the choice of weights $w$ with the optimization of the target loss, following ~\cite{TracIn}, we consider the first-order Taylor expansion of $\mathcal{L}_{\mathrm{tar}}(\theta)$ around $\theta_{t-1}$:
\begin{equation}
\label{eq: taylor}
\begin{split}
\mathcal{L}_{\mathrm{tar}}(\theta_t)
& \approx \mathcal{L}_{\mathrm{tar}}(\theta_{t-1}) + \eta_t (\theta_{t-1} - \theta_t) \nabla \mathcal{L}_{\mathrm{tar}} (\theta_{t-1}) \\
& = \mathcal{L}_{\mathrm{tar}}(\theta_{t-1}) -\eta_t \left\langle \nabla \mathcal{L}_{\mathrm{tar}}(\theta_{t-1}), \, \mathcal{P}_t ( g_t (w) ) \right\rangle.
\end{split}
\end{equation}
Under this approximation, minimizing the target loss after one update step is equivalent to maximizing the alignment between the weighted training and the target gradients under the optimizer-induced geometry. Consequently, the weight assignment problem at step $t$ can be formulated as
\begin{equation}
\label{eq: weighted objective}
\max_{w} \; \left\langle \nabla \mathcal{L}_{\mathrm{tar}}(\theta_{t-1}), \; \mathcal{P}_t ( \sum_{x_i \in \mathcal{S}^{(t)}} w_i \nabla l_i) \right\rangle
\end{equation}
With the details of proofs in Appendix \ref{[proof]: eq of 1 order}, this formulation highlights that the optimal weights depend not only on the similarity between individual training gradients and the target gradient, but also on the underlying optimizer through $P_t$. As a result, the data weighting problem is inherently optimizer-aware and evolves dynamically across training steps.
%The weight calculation depends will be detailed in the methodology Section \ref{sec:two_stage}. 
We also formulate the weight assignment problem with second-order Taylor expansion approximation in Appendix \ref{[proof]: eq of 2 order}.

\textbf{Optimizers ($\mathcal{P}_t$).} We derive the corresponding optimization goal using prevalent optimizers:
(i) \textit{Stochastic Gradient Descent (SGD)~\citep{robbins1951stochastic}.} For SGD, the update reduces to a standard gradient descent step, i.e., $\mathcal{P}_t(g) = g$. (ii) \textit{Adam optimizer~\citep{kinga2015adam}. }It  maintains exponential moving averages of both the first-order and second-order moments of the stochastic gradients, generally achieving faster and better convergence in LLM fine-tuning. Specifically, $\mathcal{P}_t(g) = \frac{ \hat{m}_t (g)} {\sqrt{\hat{v}_t (g)} + \epsilon }$, where the first and second moment estimates  are $\hat{m}_t (g) = \frac{\beta_1 m_{t-1} + (1 - \beta_1) g}{1 - \beta_1^t }$ and $\hat{v}_t (g) = \frac{\beta_2 v_{t-1} + (1 - \beta_2) g^2}{1 - \beta_2^t}$ respectively. Here, $\beta_1, \beta_2 \in [0,1)$ are decay coefficients, $\epsilon > 0$ is a small constant for numerical stability, and the square is taken element-wise. 
\squeezeup
\section{Methodology}
\label{sec:method}

We now develop a practical algorithm to optimize the step-wise 
target loss reduction objective in Eq.~\ref{eq: weighted objective} 
under the computational constraints of LLM fine-tuning. 
Section~\ref{sec:method:utility} begins with reformulating online selection 
as a \emph{subset-level gradient matching} problem, exposing why 
the additive view of batch utility is structurally inadequate. 
Building on this, Section~\ref{sec:method:ftw} 
presents our central algorithm: a decoupled 
\emph{Filter-then-Weight} procedure that solves the matching 
problem in two stages. Section~\ref{sec:target_precond} then 
integrates the Adam preconditioner into the algorithm, and 
Section~\ref{sec:efficient_gradient} concludes with the engineering 
ingredients required to scale subset-level reweighting to 
long-context LLMs.

\subsection{Subset Utility: From Additive Scoring to Gradient Matching}
\label{sec:method:utility}

At iteration $t$, we quantify the utility of a candidate mini-batch 
$\mathcal{S}^{(t)}$ by its one-step impact on the target loss. 
Approximating the target objective using a validation mini-batch 
$\mathcal{S}^{(t)}_{\mathrm{val}}$, the averaged validation gradient 
is $\nabla \ell_{\mathrm{val}} := \frac{1}{|\mathcal{S}^{(t)}_{\mathrm{val}}|} 
\sum_{x_j \in \mathcal{S}^{(t)}_{\mathrm{val}}} \nabla \ell_j$. 
Following Eq.~\ref{eq: weighted objective}, selecting and weighting 
training samples amounts to maximizing the alignment between the 
weighted training gradient and the validation gradient under the 
optimizer-induced operator:
\begin{equation}
  \max_w\; \Big\langle \nabla \ell_{\mathrm{val}},\, 
              P_t\big( {\textstyle\sum_{x_i \in \mathcal{S}^{(t)}}}
                       w_i \nabla \ell_i \big) \Big\rangle.
  \label{eq:utility}
\end{equation}

\textbf{Why additive individual scoring fails.}
When the optimizer is linear ($P_t(g) = g$) and weights are restricted 
to $w_i \in \{0, 1\}$ with $\sum_i w_i = k$, 
Eq.~\ref{eq:utility} decomposes additively across samples, 
recovering the gradient-similarity criterion of 
\citep{TracIn,xia2024less}: 
$U(\mathcal{S}^{(t)}; \mathcal{S}^{(t)}_{\mathrm{val}}) = 
\sum_{x_i \in \mathcal{S}^{(t)}} \langle \nabla \ell_{\mathrm{val}}, 
\nabla \ell_i \rangle$. 
Three structural problems emerge once we leave this restricted regime. 
First, the composite gradient is determined by both the subset and 
the weights; selection alone cannot recover the right mixture 
coefficients. Second, two samples with identical individual scores 
can be either complementary or redundant depending on their pairwise 
inner product, and individual scoring cannot distinguish the two cases. 
Third, when $P_t$ is nonlinear (as in Adam optimizer), the additive 
decomposition fails outright: the alignment of the weighted sum is 
not the sum of individual alignments.

\textbf{Reformulation as gradient matching.}
We therefore treat selection-and-weighting as \emph{subset-level 
gradient matching}: find a small set of candidates and non-negative 
weights whose composite gradient best matches the target-aligned 
update. This is naturally cast as a constrained least-squares problem,
\begin{equation}
  \min_{w \geq 0}\; 
    \big\| \nabla \ell_{\mathrm{val}} 
           - P_t\big( {\textstyle\sum_{x_i \in \mathcal{S}^{(t)}}}
                      w_i \nabla \ell_i \big) \big\|_2^2 
    + \lambda \|w\|_2^2.
  \label{eq:matching}
\end{equation}
Although gradient-matching objectives have appeared in offline coreset 
selection~\citep{killamsetty2021grad,mirzasoleiman2020coresets}, 
Eq.~\ref{eq:matching} differs in three ways that prove decisive in 
our setting: weights are constrained to be non-negative, the operator 
$P_t$ injects optimizer state directly into the objective, and—as 
developed in Section~\ref{sec:method:ftw}—the problem must be solved 
in a decoupled rather than joint manner to remain stable on LLM 
gradients.

\textbf{Why non-negativity matters.} As shown by 
\cite{slawski2013non}, unconstrained least-squares matching in 
high-dimensional spaces with co-linear features is prone to 
destructive cancellation: the solver subtracts large opposing 
vectors to fit the target, producing a composite update whose 
individual components contradict each other. In stochastic LLM 
fine-tuning this is catastrophic—the optimizer is asked to move 
away from certain training samples to satisfy the matching 
objective. Restricting $w \geq 0$ enforces constructive accumulation 
only, prunes contradictory candidates by setting their weights to 
zero, and induces sparsity such that only samples positively 
contributing to the target direction are selected. Our ablation 
(Section~\ref{sec:ablation}) confirms that lifting this 
constraint causes immediate divergence.

\textbf{Relation to alignment objectives.}
For linear optimizer mappings (e.g., SGD), Eq.~\ref{eq:matching} 
preserves the optimal solutions of common alignment objectives 
(inner-product or cosine) under mild conditions; the formal statement 
and proof are given in Appendix~\ref{sec:inner_cosine}. For 
nonlinear optimizers such as Adam/AdamW, strict equivalence no longer 
holds, but Eq.~\ref{eq:matching} remains geometrically meaningful: 
it matches effective update directions in the optimizer-induced 
feature space, serving as a principled surrogate for optimizer-aware 
online selection.

\textbf{Second-order interpretation.}
Eq.~\ref{eq:matching} also admits a second-order interpretation. 
Approximating the target Hessian by an isotropic operator, the 
second-order Taylor expansion of the target loss reduces, up to a 
weight-norm penalty, to the gradient-matching objective in 
Eq.~\ref{eq:matching}; the derivation is given in 
Appendix~\ref{appendix: connection}. The quadratic term arising from this 
expansion introduces \emph{pairwise interactions} among selected 
samples, naturally penalizing redundant gradients and producing the 
diminishing-returns behavior that the additive view cannot express. 

\subsection{Filter-then-Weight: A Decoupled Two-Stage Algorithm}
\label{sec:method:ftw}

Solving Eq.~\ref{eq:matching} directly through coupled 
selection-and-weighting—e.g., Orthogonal Matching Pursuit 
(OMP)~\citep{killamsetty2021grad}, which re-solves the weights every 
time a new sample is added—is the natural choice and the one taken 
by prior gradient-matching work. We find this coupled formulation 
\emph{fails on LLM gradients}: the per-step weight solver chases 
high-frequency noise from stochastic sampling, random projection, 
and limited validation batches, producing unstable selections. The 
same behavior is visible in our baseline, \textsc{Grad-Match}~\citep{killamsetty2021grad} (Section~\ref{subsec: main}).

We therefore decouple the two operations into a robust 
\emph{filtering} stage that identifies a candidate backbone using 
unit weights, and a precise \emph{weighting} stage that solves for 
the final coefficients only over the filtered set. The intuition is 
that filtering with unit weights is a low-frequency operation 
insensitive to small gradient perturbations, while the weighting 
stage operates on a much smaller and well-conditioned problem—each 
stage is exposed to the noise level it can handle. We build on the 
expansion of Eq.~\ref{eq:matching} into a precomputed Alignment 
Vector $\mathbf{b}$ and Gram Matrix $\mathbf{G}$ (detailed in 
Appendix~\ref{section: omp}), which reduces both stages to 
small-scale operations on these reusable components.

\textbf{Stage 1: Greedy residual filtering.}
Given a candidate pool $\mathcal{S}^{(t)}$ and a budget 
$B_{\mathrm{tr}}$, we initialize a residual 
$r \leftarrow \nabla\tilde\ell_{\mathrm{val}}$ (the preconditioned 
target; see Section~\ref{sec:target_precond}) and iteratively select 
the candidate whose gradient most reduces the residual:
\begin{equation}
  i^* = \arg\max_{i \notin \mathcal{S}'}\; 
        \langle \nabla \ell_i,\, r \rangle, 
  \qquad
  r \leftarrow r - \nabla \ell_{i^*},
  \label{eq:greedy_filter}
\end{equation}
until $|\mathcal{S}'| = B_{\mathrm{tr}}$. This is OMP \emph{without} 
the inner weight-resolve step: by holding weights at unity, we rank 
candidates purely by their geometric contribution to the residual, 
prioritizing diversity over fit. The procedure is robust because no 
small noisy quantity is being inverted at each step.

\textbf{Stage 2: Non-negative least squares reweighting.}
Given the filtered set $\mathcal{S}'$, we extract the corresponding 
sub-vector $\mathbf{b}_{\mathcal{S}'}$ and sub-matrix 
$\mathbf{G}_{\mathcal{S}'}$ and solve the resulting non-negative least 
squares (NNLS) problem:
\begin{equation}
  w^* = \arg\min_{w \geq 0}\; 
        w^\top \mathbf{G}_{\mathcal{S}'}\, w 
        - 2\, \mathbf{b}_{\mathcal{S}'}^\top w + \lambda \|w\|_2^2.
  \label{eq:nnls}
\end{equation}
Because $|\mathcal{S}'|$ equals the per-step batch size, this 
quadratic program is negligible relative to gradient computation. 
The non-negativity constraint follows Section~\ref{sec:method:utility}, while the global 
optimization over the fixed set ensures the composite gradient 
maintains high fidelity to the optimizer-aware target. The composite 
gradient $\bar\nabla \ell^{(t)} = \sum_{i \in \mathcal{S}'} w_i^* 
\nabla \ell_i$ is then used to update the model. The complete in-step 
training procedure is detailed in 
Appendix~\ref{section: our pseudocodes}, Algorithm~\ref{alg: ft w. ds}.

\textbf{Why decoupling is the critical design choice.}
The two stages address different failure modes of the coupled solver. 
Stage 1 commits early to a geometrically diverse backbone, 
eliminating the joint solver's tendency to pick samples that cancel 
one another's noise. Stage 2 then refines coefficients globally over 
a fixed, small set, allowing precise alignment with the target 
without the instability of greedy weight updates. Our ablations 
(Section~\ref{sec:ablation}) show two consequences: (i) coupling the 
stages, as in standard OMP, recovers the unstable behavior of 
\textsc{Grad-Match}; (ii) replacing the greedy filter with top-$k$ 
selection achieves comparable peak performance but degrades over the 
training trajectory because top-$k$ admits redundant samples. Both 
findings indicate that decoupling—rather than any one of its 
components in isolation—is what makes subset-level reweighting 
tractable on LLMs.

\subsection{Optimizer-Aware Preconditioning}
\label{sec:target_precond}

The Filter-then-Weight algorithm in Section~\ref{sec:method:ftw} 
requires a preconditioned target $\nabla\tilde\ell_{\mathrm{val}}$ 
that reflects the optimizer's effective update geometry. For Adam, 
$\mathcal{P}_t(g) = \hat{m}_t(g) / (\sqrt{\hat{v}_t(g)} + \epsilon)$ 
depends on the candidate-dependent second moment $\hat{v}_t$, making 
Eq.~\ref{eq:matching} nonlinear in $w$ and intractable to optimize 
exactly. We adopt a \emph{linearized Adam} approximation that 
freezes the preconditioner at the previous-step state:
\begin{equation}
\label{eq:linearized_adam}
\mathcal{P}_t(g) \approx \mathbf{D}_{t-1} \odot g,  
\quad \mathbf{D}_{t-1} = 
\frac{1-\beta_1}{(1-\beta_1^t)(\sqrt{\hat{v}_{t-1}} + \epsilon)},
\end{equation}
which is well-justified at fine-tuning time scales because 
$\hat{v}_t$ is dominated by historical accumulation and changes 
slowly between adjacent steps.

\textbf{Efficient implementation via target-side absorption.}
A natural design question is whether $\mathbf{D}_{t-1}$ should 
appear only in the alignment vector $\mathbf{b}$ (first-order 
preconditioning), or also in the Gram matrix $\mathbf{G}$ 
(second-order preconditioning). In both cases, the diagonal 
structure of $\mathbf{D}_{t-1}$ admits an efficient implementation: 
all preconditioning can be absorbed into the validation side, 
keeping the candidate-side computation entirely in raw form. For 
the first-order term,
\begin{equation}
\label{eq:adjoint_transfer}
 \langle \nabla\ell_{\mathrm{val}},\, 
          \mathbf{D}_{t-1} \odot \nabla\ell_i \rangle
  = \langle \mathbf{D}_{t-1} \odot \nabla\ell_{\mathrm{val}},\, 
        \nabla\ell_i \rangle,
\end{equation}
so $\mathbf{D}_{t-1} \odot \nabla\ell_{\mathrm{val}}$ can be 
precomputed once per step. For the second-order Gram term, the 
identity
\begin{equation}
\label{eq:second_order_pull}
\langle \mathbf{D}_{t-1} \odot \nabla\ell_i,\, 
        \mathbf{D}_{t-1} \odot \nabla\ell_j \rangle 
= \langle \nabla\ell_i,\, 
          \mathbf{D}_{t-1}^2 \odot \nabla\ell_j \rangle
\end{equation}
similarly absorbs the two-sided preconditioning into a single 
target-side rescaling $\mathbf{D}_{t-1}^2 \odot 
\nabla\ell_j$, computed once per validation sample per step. 
This 
target-side absorption avoids any per-candidate rescaling and 
makes the two preconditioning variants computationally equivalent. 

% Our ablation 
% (Section~\ref{sec:ablation}) finds the additional second-order 
% rescaling to have a minor effect on downstream performance, so we 
% adopt the simpler first-order form by default:
% \begin{equation}
% \mathbf{b}_i = \langle \mathbf{D}_{t-1} \odot 
% \nabla\ell_{\mathrm{val}},\, \nabla\ell_i \rangle,
% \qquad
% \mathbf{G}_{ij} = \langle \nabla\ell_i,\, \nabla\ell_j \rangle.
% \end{equation}

\textbf{Compatibility with random projection.}
The historical second moment $\hat{v}_{t-1}$ cannot be directly 
projected because squaring is nonlinear under random projection. 
We instead maintain a separate second-moment estimate computed 
from projected gradients using the same projection matrix, 
ensuring that $\mathbf{D}_{t-1}$ in Eq.~\ref{eq:linearized_adam} 
shares the same space as the projected alignment computation. The 
additional cost is negligible.

\begin{table}[b]
\centering
\small
\caption{Time and space complexity comparison for sequential data in a batch.}
\label{tab:complexity_comparison}
\setlength{\tabcolsep}{6pt}
\resizebox{0.9\textwidth}{!}{%
\begin{tabular}{lccc}
\toprule
 & \textbf{Naive} & \textbf{Ghost} & \textbf{Ours} \\
\midrule
\textbf{Time} 
& $\mathcal{O}(LT^2B_{tr}B_{val}d_1d_2)$
& $\mathcal{O}(LT^2B_{tr}B_{val}(d_1+d_2))$
& $\mathcal{O}(LTB_{tr}d_2(B_{val}d_1 + T))$ \\[6pt]
\textbf{Space} 
& $\mathcal{O}(LB_{tr}B_{val}d_1d_2)$
& $\mathcal{O}(LT(B_{tr}+B_{val})d_1 + T^2B_{tr}B_{val})$
& $\mathcal{O}(LT(B_{tr}+B_{val})d_1 + B_{tr}d_2\max(d_1,T))$ \\
\bottomrule
\end{tabular}
}
\end{table}

\subsection{Efficient Gradient Computation for LLMs}
\label{sec:efficient_gradient}

Computing the inner products required by $\mathbf{b}$ and 
$\mathbf{G}$ over full LLM gradients $\nabla\ell_i \in \mathbb{R}^n$ 
is infeasible: $n$ is large, and per-sample gradients on long-context 
inputs further inflate intermediate tensors along the sequence 
dimension $T$. We address both with a sequence of structural 
simplifications, ending in a representation whose cost scales 
favorably with $T$.

\textbf{LoRA + random projection.}
We fine-tune with LoRA~\citep{hu2022lora}, which freezes the 
backbone and injects small low-rank adapters, reducing the effective 
gradient dimension to $\sim$1\% of the full model. We further apply 
random projection~\citep{xia2024less, wangcapturing} with matrix 
$\mathbf{R} \in \mathbb{R}^{d \times n}$, $d \ll n$, drawn from a 
Gaussian distribution. By the Johnson–Lindenstrauss 
lemma~\citep{johnson1984extensions}, projected inner products 
approximately preserve the originals, $\langle \mathbf{Ru}, 
\mathbf{Rv}\rangle \approx \langle \mathbf{u}, \mathbf{v}\rangle$, at 
substantially reduced cost. To keep the candidate pool 
$\mathcal{S}^{(t)}$ within GPU memory, we further partition it into 
mini-batches $\mathcal{S}^{(t,c)}$, $c = 1, \ldots, \alpha$, 
processed in a manner analogous to gradient accumulation.

\textbf{Factorized rank-1 representation.}
For a linear layer with weight $W \in \mathbb{R}^{d_1 \times d_2}$ 
(including LoRA's down/up projections), the per-sample gradient 
decomposes as the outer product of input activation 
$a_i \in \mathbb{R}^{d_1}$ and pre-activation gradient 
$g_i \in \mathbb{R}^{d_2}$: $\nabla\ell_i = g_i a_i^\top$. 
By the Kronecker product property, inner products factorize as
\begin{equation}
\label{eq: 1-order dotprod}
\langle \nabla\ell_i, \nabla\ell_j \rangle 
= (g_j^\top g_i)(a_j^\top a_i),
\end{equation}
the \emph{ghost dot-product} of \citep{bu2023differentially, 
wang2025capturing}, which avoids materializing the full $d_1 \times 
d_2$ gradient matrix. Crucially, by working with the rank-1 factors 
directly, projection cost scales as $d_1 + d_2$ rather than 
$d_1 d_2$, preserving substantially more information per byte of 
memory than projecting the full gradient.

\textbf{Sequence-length-aware reordering.}
For sequences of length $T$, the rank-1 factors become matrices 
$a_i \in \mathbb{R}^{d_1 \times T}$ and $g_i \in \mathbb{R}^{d_2 
\times T}$, so naive evaluation of Eq.~\ref{eq: 1-order dotprod} 
materializes $T \times T$ intermediate matrices of cost 
$\mathcal{O}(T^2)$ per pair, dominating the budget on long-context 
LLMs. With training and validation batch sizes $B_{tr} = 
|\mathcal{S}^{(t)}|$ and $B_{val} = |\mathcal{S}_{\text{val}}^{(t)}|$, 
this scales to prohibitive $TB_{val} \times TB_{tr}$ memory. We 
reorder the computation to aggregate over the validation batch 
\emph{before} touching individual training samples:
\begin{equation}
\label{eq:reordered inner-prod}
\langle \nabla\ell_i, \nabla\ell_{\mathrm{val}} \rangle 
= \Big\langle \hat{g}_i,\; \hat{a}_i 
   \Big( \sum_{x_j \in \mathcal{S}_{\text{val}}^{(t)}} 
         \hat{a}_j \hat{g}_j^\top \Big) \Big\rangle,
\end{equation}
where $\hat{a}, \hat{g}$ denote projected factors. The bracketed sum 
is a reusable $k \times k$ matrix computed once per step, collapsing 
the dependence on $T$ from quadratic to linear. When $B_{tr}$ and $T$ grow large, the 
Ghost formulation OOMs while ours sustains a substantially higher 
ceiling—at which point the binding constraint shifts to model-side 
kernels (e.g., the cross-entropy logits buffer) rather than the 
score tensor itself. We report the full empirical analysis in 
Appendix~\ref{sec:memory_frontier}.

% \textbf{Sequence-length-aware reordering.}
% For sequences of length $T$, the rank-1 factors become matrices 
% $a_i \in \mathbb{R}^{d_1 \times T}$ and $g_i \in \mathbb{R}^{d_2 
% \times T}$, so naive evaluation of Eq.~\ref{eq: 1-order dotprod} 
% materializes $T \times T$ intermediate matrices of cost 
% $\mathcal{O}(T^2)$ per pair, dominating the budget on long-context 
% LLMs. With training and validation batch sizes $B_{tr}$ and $B_{val}$, 
% this scales to prohibitive $TB_{val} \times TB_{tr}$ memory. We 
% reorder the computation to aggregate over the validation batch 
% \emph{before} touching individual training samples:
% \begin{equation}
% \label{eq:reordered inner-prod}
% \langle \nabla\ell_i, \nabla\ell_{\mathrm{val}} \rangle 
% = \Big\langle \hat{g}_i,\; \hat{a}_i 
%    \Big( \sum_{x_j \in \mathcal{S}_{\text{val}}^{(t)}} 
%          \hat{a}_j \hat{g}_j^\top \Big) \Big\rangle,
% \end{equation}
% where $\hat{a}, \hat{g}$ denote projected factors. The bracketed sum 
% is a reusable $k \times k$ matrix computed once per step, collapsing 
% the dependence on $T$ from quadratic to linear.
% Table~\ref{tab:complexity_comparison} compares per-batch costs of 
% the naive, ghost, and our implementations; detailed derivations are 
% in Appendix~\ref{sec: Complexity Comparison}. Our method is strictly 
% better than ghost in both dimensions whenever $d_1, d_2 < T$—a 
% regime that holds for modern long-context LLMs and is further 
% enforced by random projection.

\squeezeup
\section{Experiments}
\label{sec: exp}
\squeezeup
We evaluate the proposed optimizer-aware online data selection framework on two representative instruction-following benchmark datasets. 
Our experiments aim to answer the following questions: 1. \textbf{Effectiveness}: Can the proposed framework improve performance compared to existing data selection methods? (Section \ref{subsec: main} \#1-3) 2. \textbf{Data efficiency}: How does the method behave under strict training budget constraints? (Section \ref{subsec: main} \#4)  3.\textbf{Mechanism}: Which components of the framework contribute most to the observed gains? (Section \ref{sec:ablation})
To this end, we conduct comparisons under both best-of-run and fixed-budget settings, followed by detailed ablation studies and sensitivity analyses. %Codes are provided in \url{https://anonymous.4open.science/r/FWDS/}.

\begin{table*}[t]
\centering
\small
\setlength{\tabcolsep}{6pt}
\caption{\textbf{Best-of-run performance.} For each seed we select 
the best checkpoint per method and report mean\std{std} over seeds. 
\textbf{Bold} marks the best result, \underline{underline} the 
second best.}
\label{tab:best_of_run}
\resizebox{\textwidth}{!}{
\begin{tabular}{ll|cc|ccc|c|c}
\toprule
& & \multicolumn{2}{c|}{\emph{References}} & \multicolumn{3}{c|}{\emph{Selection only}} & \emph{Coupled} & \emph{Ours} \\
Dataset & Backbone & Origin & Full & TracIn & LESS & GREATS & GRAD-MATCH & Filter-then-Weight \\
\midrule
\multirow{2}{*}{MMLU (Acc)}
& Llama-3.2-1B
& 28.13 \std{0.00}
& 29.09 \std{1.01}
& 29.24 \std{0.24}
& 29.74 \std{0.07}
& 30.07 \std{0.37}
& \underline{30.14 \std{0.41}}
& \textbf{30.93 \std{0.03}} \\
& Qwen3-0.6B
& 45.86 \std{0.00}
& \textbf{46.64 \std{0.17}}
& 46.45 \std{0.08}
& 46.56 \std{0.28}
& 46.06 \std{0.42}
& 45.72 \std{0.62}
& \underline{46.63 \std{0.40}} \\
\midrule
\multirow{2}{*}{TyDiQA (F1)}
& Llama-3.2-1B
& 12.69 \std{0.54}
& 42.47 \std{0.30}
& 47.01 \std{1.43}
& 47.21 \std{2.58}
& \underline{47.50 \std{0.56}}
& 45.35 \std{1.35}
& \textbf{48.86 \std{1.67}} \\
& Qwen3-0.6B
& 20.94 \std{0.68}
& 43.65 \std{0.61}
& 45.17 \std{0.97}
& 45.07 \std{0.95}
& \underline{45.25 \std{0.43}}
& 44.23 \std{1.33}
& \textbf{46.29 \std{0.62}} \\
\bottomrule
\end{tabular}
}
\end{table*}

\begin{table*}[t]
\centering
\small
\setlength{\tabcolsep}{6pt}
\caption{\textbf{Fixed-budget performance.} All methods evaluated 
at the final checkpoint of the $5\%$ data budget, formatting as in Table~\ref{tab:best_of_run}.}
\label{tab:fixed_budget}
\resizebox{\textwidth}{!}{
\begin{tabular}{ll|cc|ccc|c|c}
\toprule
& & \multicolumn{2}{c|}{\emph{References}} & \multicolumn{3}{c|}{\emph{Selection only}} & \emph{Coupled} & \emph{Ours} \\
Dataset & Backbone & Origin & Full & TracIn & LESS & GREATS & GRAD-MATCH & Filter-then-Weight \\
\midrule
\multirow{2}{*}{MMLU (Acc)}
& Llama-3.2-1B
& 28.13 \std{0.00}
& 28.67 \std{0.63}
& 29.04 \std{0.04}
& \underline{29.67 \std{1.31}}
& 29.59 \std{0.40}
& 28.77 \std{0.56}
& \textbf{30.93 \std{0.03}} \\
& Qwen3-0.6B
& 45.86 \std{0.00}
& \textbf{46.59 \std{0.10}}
& 46.16 \std{0.07}
& 45.98 \std{0.37}
& 46.06 \std{0.42}
& 44.98 \std{0.72}
& \underline{46.23 \std{0.67}} \\
\midrule
\multirow{2}{*}{TyDiQA (F1)}
& Llama-3.2-1B
& 12.69 \std{0.54}
& 38.68 \std{1.36}
& 47.01 \std{1.34}
& 46.88 \std{0.91}
& \underline{47.16 \std{1.08}}
& 44.15 \std{0.85}
& \textbf{48.67 \std{2.68}} \\
& Qwen3-0.6B
& 20.94 \std{0.68}
& 41.14 \std{0.64}
& 43.07 \std{0.97}
& \underline{43.57 \std{0.28}}
& 42.70 \std{0.92}
& 43.14 \std{1.13}
& \textbf{44.82 \std{1.00}} \\
\bottomrule
\end{tabular}
}
\vspace{-2mm}
\end{table*}

\squeezeup
\subsection{Experimental Setup}
\squeezeup
We construct the training corpus using the Open-Instruct 
framework~\citep{lambert2024tulu3}, which aggregates several widely 
used instruction-tuning datasets including FLAN 
v2~\citep{longpre2023flan}, 
Chain-of-Thought~\citep{wei2022chain}, 
Dolly~\citep{conover2023dolly}, and 
OASST1~\citep{kopf2023openassistant}, yielding approximately 270k 
instruction-response pairs in total. To prevent evaluation leakage, 
we remove any training samples overlapping with the target 
benchmarks. We evaluate on two benchmarks reflecting complementary 
aspects of instruction-following: \textbf{MMLU}~\citep{hendrycksmeasuring}, 
a broad knowledge benchmark covering 57 academic subjects, on which 
we report 5-shot accuracy averaged across all tasks; and 
\textbf{TyDiQA}~\citep{clark2020tydi}, a multilingual 
question-answering benchmark covering 11 typologically diverse 
languages, on which we report 1-shot macro-averaged F1.

Experiments use two instruction-tunable backbones at different 
scales: Llama-3.2-1B and Qwen3-0.6B. All selection-based methods 
are constrained to a strict $5\%$ data budget. At each step, we 
draw a batch of size $b_{tr}$ from a candidate pool of size 
$B_{tr} = \alpha b_{tr}$ with oversampling factor $\alpha = 4$; 
validation gradients are estimated analogously from a pool of size 
$B_{val} = \alpha b_{val}$. Due to GPU memory constraints, batch 
configurations differ across backbones: $b_{tr}{=}8, b_{val}{=}4$ 
for Llama-3.2-1B and $b_{tr}{=}6, b_{val}{=}2$ for Qwen3-0.6B. 
Models are evaluated every 200 steps, and training terminates once 
the cumulative number of processed samples reaches the $5\%$ 
budget. Additional implementation details are in 
Appendix~\ref{subsec: setting}.

% \subsection{Baselines}
% We compare our method against representative baselines spanning both full-data training and gradient-based data selection strategies. Notably, baseline methods are designed only for data selection, except that the third method GRAD-MATCH~\citep{killamsetty2021grad} is built with weighting mechanism.
% \textbf{(i) Origin:}
% The pretrained backbone evaluated directly without instruction tuning, representing the zero-shot baseline.
% \textbf{(ii) Full Data:}
% Standard supervised fine-tuning using all available samples. 
% This baseline represents the upper bound in terms of data usage but requires significantly higher computational cost.
% \textbf{(iii) GRAD-MATCH}~\citep{killamsetty2021grad}:
% An online subset selection method that jointly performs sample selection and weighting using Orthogonal Matching Pursuit (OMP).
% \textbf{(iv) TracIn} ~\citep{TracIn}:
% An influence-based method estimating sample importance via gradient similarity between training and validation examples. 
% We adapt it to the online setting by computing gradients on-the-fly without storing checkpoints.
% \textbf{(v) LESS}~\citep{xia2024less}:
% A gradient embedding matching method that incorporates Adam optimizer statistics for gradient preconditioning.
% We adapt it to the online setting by removing checkpoint-based gradient accumulation.
% \textbf{(vi) GREATS}~\citep{wang2024greats}:
% A greedy online selection method that explicitly models intra-batch interactions among selected samples.
% ep, explicitly modeling intra-batch interactions.
\squeezeup
\subsection{Baselines}
\squeezeup
We compare against representative baselines spanning full-data 
training and gradient-based selection strategies. Among them, only 
GRAD-MATCH performs both selection and weighting; the others apply 
uniform weights to selected samples.
\textbf{(i) Origin.} The pretrained backbone evaluated directly 
without instruction tuning, representing the zero-shot lower bound. 
\textbf{(ii) Full Data.} Standard supervised fine-tuning on all 
available samples, representing the full-budget upper bound on data 
usage. \textbf{(iii) TracIn}~\citep{TracIn}. An influence-based 
method scoring each sample by gradient similarity to the validation 
set; we adapt it online by computing gradients on-the-fly without 
checkpoint storage. \textbf{(iv) LESS}~\citep{xia2024less}. 
A gradient-embedding matching method incorporating Adam optimizer 
statistics; we adapt it online by removing checkpoint-based 
gradient accumulation. \textbf{(v) GREATS}~\citep{wang2024greats}. 
A greedy online selector that explicitly models intra-batch 
interactions among selected samples. \textbf{(vi) GRAD-MATCH}~\citep{killamsetty2021grad}.
A coupled selection-and-weighting method using Orthogonal Matching 
Pursuit (OMP), serving as the closest prior baseline to our 
formulation.

\subsection{Main Results}
\label{subsec: main}

Tables~\ref{tab:best_of_run} and~\ref{tab:fixed_budget} report 
results under two complementary evaluation protocols. 
\textbf{Best-of-run evaluation} (Table~\ref{tab:best_of_run}) 
reports the best checkpoint per seed, measuring peak achievable 
performance. \textbf{Fixed-budget evaluation} 
(Table~\ref{tab:fixed_budget}) reports the final checkpoint at the 
$5\%$ data budget, providing a fair comparison under identical 
training resources. Filter-then-Weight achieves the strongest 
results on both backbones across both metrics on TyDiQA and on 
Llama-3.2-1B for MMLU; on Qwen3-0.6B/MMLU it is competitive with 
the strongest baseline. We discuss the comparisons by baseline 
category below.

\textbf{Versus selection-only methods.}
TracIn and LESS rely on individual gradient-similarity scoring 
without modeling sample interactions; GREATS partially addresses 
this through a greedy intra-batch correction. None of the three, 
however, refines the weights of selected samples. Adding our 
subsequent NNLS reweighting stage allows the resulting batch 
gradient to better approximate the target direction, which 
empirically translates into consistent gains across most 
configurations, with the largest margins on TyDiQA. This is 
consistent with the prediction in 
Section~\ref{sec:method:utility} that additive scoring 
underexploits the candidate pool when sample interactions matter.

\textbf{Versus coupled selection-weighting.}
GRAD-MATCH is the only baseline that also performs reweighting; it 
does so by re-solving the weights at every greedy step (OMP). On 
LLM gradients we observe this coupled formulation to be markedly 
unstable: GRAD-MATCH consistently underperforms even the 
selection-only baselines on TyDiQA, despite using a more 
expressive weighting mechanism. Our decoupled Filter-then-Weight 
recovers the benefit of reweighting without the instability, 
directly validating the design choice analyzed in 
Section~\ref{sec:method:ftw}: separating noisy filtering from 
precise NNLS weighting exposes each stage to the noise level it 
can handle.

\textbf{Versus full-data training.}
Several selection methods, including ours, outperform full-data 
training despite using only $5\%$ of the corpus, confirming the 
value of removing noisy or off-target samples. The exception is 
Qwen3-0.6B/MMLU, where Full remains competitive: both Open-Instruct 
and MMLU are general-purpose corpora, so the candidate pool 
already aligns well with the target distribution and aggressive 
subset selection offers little marginal benefit.

\textbf{Data efficiency.}
To probe the regime where selection matters most, 
Figure~\ref{fig:data_ratio_f1} plots TyDiQA F1 as a function of the cumulative training data ratio up to the $5\%$ budget. All selection methods substantially outperform Origin from early 
in training, indicating that a small set of informative instruction  examples is sufficient for non-trivial improvement. Filter-then-Weight maintains the lead across most ratios, with the 
gap widening at smaller budgets—precisely the regime where 
informative-sample selection matters most.

\begin{figure}[t]
    \centering
    \begin{minipage}[t]{0.4\columnwidth}
        \centering
        \includegraphics[width=\linewidth]{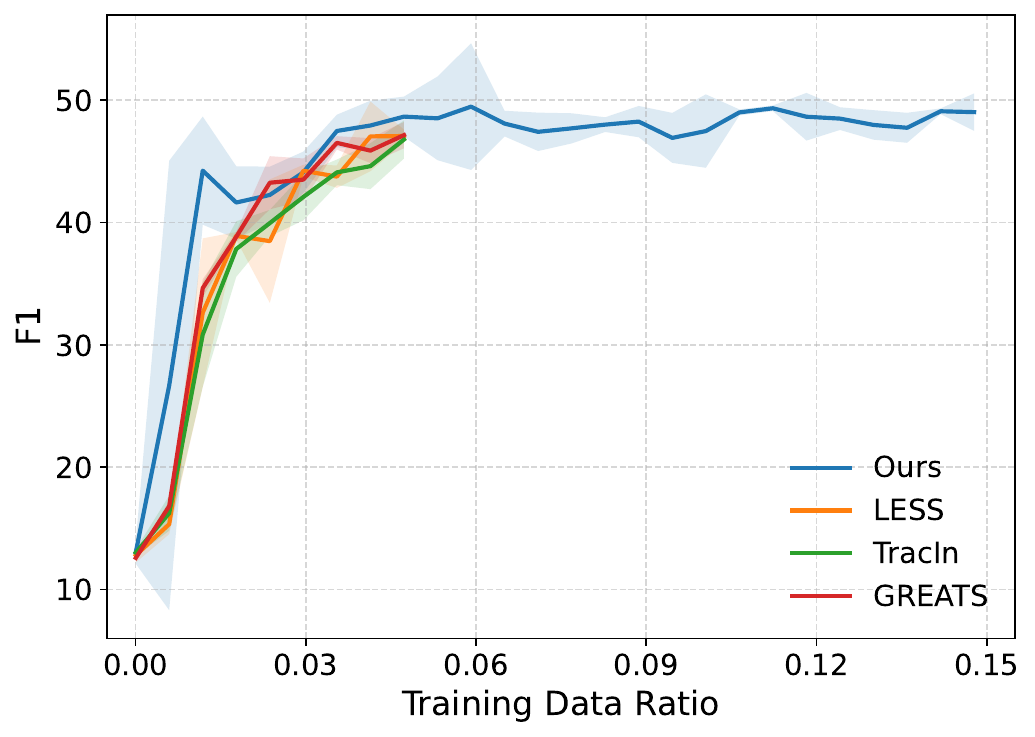}
        \vspace{-2mm}
        \caption{TyDiQA performance (F1) as a function of the training data ratio.}
        \label{fig:data_ratio_f1}
    \end{minipage}
    \hfill
    \begin{minipage}[t]{0.55\columnwidth}
        \centering
        \includegraphics[width=\linewidth]{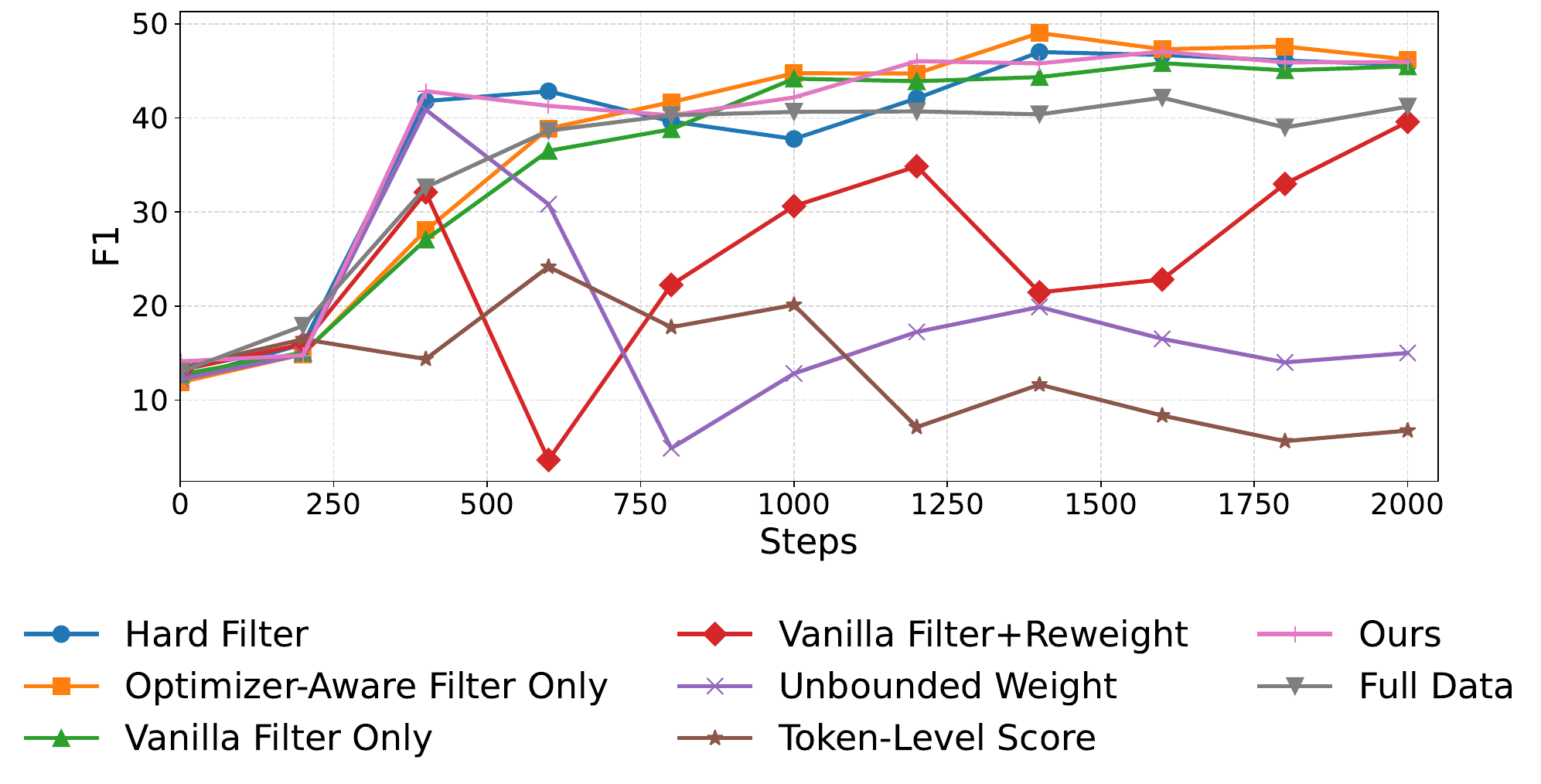}
        \vspace{-2mm}
        \caption{Preliminary ablation (single seed) used to prune the design space of our proposed framework.}
        \label{fig:ablation_f1}
    \end{minipage}
    \vspace{-3mm}
\end{figure}
\squeezeup

\begin{table}[t]
\vspace{-3mm}
\centering
\small
\begin{minipage}[t]{0.5\columnwidth}
\centering
\caption{Multi-seed ablation results on the selected variants (mean \std {std}).}
\label{tab:ablation_studies}
\setlength{\tabcolsep}{4pt}
\resizebox{\linewidth}{!}{%
\begin{tabular}{lcc}
\toprule
\textbf{Method} & \textbf{Best F1 $\leq$2000} & \textbf{F1 @2000} \\
\midrule
1. Hard Filter+Reweight
& \underline{48.72 \std {1.86}}
& 45.08 \std {1.17} \\
2. Optimizer-Aware Filter Only
& 47.21 \std {2.58}
& 46.88 \std {0.91} \\
3. Vanilla Filter Only
& 47.01 \std {1.43}
& \underline{47.01 \std {1.43}} \\
4. Vanilla Filter+Reweight
& 46.39 \std {6.13}
& 45.69 \std {5.34} \\
\midrule
Ours
& \textbf{48.86 \std {1.67}}
& \textbf{48.67 \std {2.68}} \\
\bottomrule
\end{tabular}
}
\end{minipage}
\hfill
\begin{minipage}[t]{0.35\columnwidth}
\centering
\caption{Sensitivity to the random projection dimension $d$ on MMLU dataset.}
\label{tab:param_dim_mean_std}
\setlength{\tabcolsep}{5pt}
\resizebox{\linewidth}{!}{%
\begin{tabular}{c|cc}
\toprule
$d$ & Best Acc. $\le$1600 & Acc. @1600 \\
\midrule
32  & 30.04 \std {0.20} & 29.66 \std {0.75} \\
64  & 30.21 \std {0.49} & 28.36 \std {0.57} \\
128 & 30.17 \std {0.26} & 28.67 \std {0.39} \\
256 & \textbf{31.17 \std{0.13}} & \textbf{30.16 \std {1.56}} \\
\bottomrule
\end{tabular}
}
\end{minipage}
\vspace{-5mm}
\end{table}

\subsection{Ablation Studies and Hyperparameter Analysis}
\label{sec:ablation}

We perform mechanistic ablations on Llama-3.2-1B/TyDiQA to validate 
the contribution of individual design choices in 
Filter-then-Weight. The ablations span filter mechanism (greedy 
vs.\ top-$k$), presence of the reweighting stage, optimizer-aware 
preconditioning, weight constraints (non-negative vs.\ unbounded), 
and gradient aggregation granularity (sequence vs.\ token level). 
Figure~\ref{fig:ablation_f1} shows preliminary single-seed runs 
across all variants: \textit{Unbounded Weights} and 
\textit{Token-Level Score} diverge or collapse and are excluded 
from the multi-seed comparison. We discuss them first, then turn 
to the four well-behaved variants reported in 
Table~\ref{tab:ablation_studies}; full variant definitions are in 
Appendix~\ref{app:ablation_variants}.

\textbf{Failure modes confirm two design constraints.}
\textit{Unbounded Weights}, which lifts the non-negative constraint 
in Eq.~\ref{eq:matching}, diverges within the first few 
hundred steps. The unconstrained solver exploits gradient 
co-linearity by subtracting opposing vectors—forcing the optimizer 
to step away from certain training samples—which destabilizes 
fine-tuning, validating the analysis in 
Section~\ref{sec:method:utility} and the necessity of the NNLS 
solver. \textit{Token-Level Score}, which restricts gradient 
similarity to identical token positions, also collapses, confirming 
that gradient aggregation must respect sequence-level semantic 
structure rather than per-position alignment.

\textbf{Reweighting requires optimizer-aware preconditioning.}
Comparing \emph{Vanilla Filter Only} and \emph{Vanilla 
Filter+Reweight} in Table~\ref{tab:ablation_studies} isolates the 
effect of NNLS reweighting under raw (un-preconditioned) gradients. 
Surprisingly, adding the reweighting stage on raw gradients makes 
performance worse: the weight solver overfits to 
high-variance directions misaligned with the optimizer's actual 
update geometry. Once the alignment vector is preconditioned 
(Ours), reweighting becomes beneficial. This contrast directly 
supports the asymmetric preconditioning design in 
Section~\ref{sec:target_precond}: the alignment side requires 
optimizer geometry to expose the correct descent direction, without 
which the additional capacity of the weighting stage becomes 
counterproductive.

\textbf{The reweighting stage itself adds non-trivial value.}
\emph{Optimizer-Aware Filter Only} shares our preconditioning and 
filtering but omits the NNLS stage, so comparing it to Ours isolates 
the marginal contribution of subset-level reweighting. The gap 
confirms that even with an optimizer-aware filter, refining 
coefficients to form a precise composite update is necessary for 
peak performance, in line with the subset-level matching argument 
in Section~\ref{sec:method:utility}.

\textbf{Decoupling matters: greedy filter outperforms top-$k$ over 
the trajectory.}
\emph{Hard Filter+Reweight} and Ours both employ NNLS reweighting 
and differ only in the filter mechanism (top-$k$ vs.\ greedy). Both 
achieve comparable \emph{peak} F1, but top-$k$ degrades by the 
final checkpoint while greedy maintains its performance. Top-$k$ 
admits redundant samples that the NNLS solver must subsequently 
down-weight to zero, wasting batch capacity; greedy filtering 
ensures geometric diversity at admission time. Together with the 
GRAD-MATCH comparison in Section~\ref{subsec: main}, this confirms 
the design analyzed in Section~\ref{sec:method:ftw}: it is the 
\emph{decoupling} of diverse filtering from precise weighting—not 
any one of these components in isolation—that drives the gain.

\textbf{Sensitivity to projection dimension.}
Table~\ref{tab:param_dim_mean_std} reports MMLU accuracy as the 
random projection dimension $d$ varies. Both peak and late-stage 
accuracy improve monotonically with $d$, with $d{=}256$ giving the 
strongest and most stable results across seeds.
\squeezeup
\section{Conclusion}
\label{sec: conclusion}

We have introduced \emph{Filter-then-Weight}, a framework for online data selection and reweighting in LLM fine-tuning that reformulates the problem as subset-level gradient matching with non-negative weights. We model batch utility as a 
\emph{composite} update whose alignment with the target depends jointly on which samples are selected and how they are weighted, and decouple this into a robust filtering stage and a precise weighting stage—a design choice that, our ablations show, is itself the principal driver of gains on noisy LLM gradients. To keep the framework practical for long context, we further develop a factorized rank-1 gradient representation with sequence-length-aware reordering. Across different settings of models and datasets, Filter-then-Weight consistently improves over leading online selection baselines under the same data budget. 

\textbf{Limitations.} Natural extensions of this paper include 
reusing a bounded buffer of historical gradient trajectories 
across steps. This remains within the online constraint while 
trading memory and time for richer training dynamics, 
analogous to the common practices in offline 
setting~\citep{xia2024less,deng2024influential}. Another direction 
is to explore richer second-order surrogates, as detailed in our analysis (Appendix~\ref{[proof]: eq of 2 order}).

\newpage

% ============================
% Acknowledgments
% ============================

% \begin{ack}
% \end{ack}

% ============================
% Bibliography
% ============================

\bibliographystyle{plainnat}
\bibliography{ref}

% ============================
% Appendix
% ============================

\newpage
\appendix

\section{Algorithm Details}

\subsection{Orthogonal Matching Pursuit (OMP)}
\label{section: omp}

\begin{algorithm}[b]
\caption{\texttt{OMP-Select Sub-routine}}
\label{alg: omp sub}
\begin{flushleft}
    \textbf{Input:} Candidate set $\mathcal{S}^{(t)}$; Gram matrix 
    $\mathbf{G}$; alignment vector $\mathbf{b}$; gradients 
    $\{\nabla\ell_i\}$; preconditioned target 
    $\nabla\tilde\ell_{\mathrm{val}}$; budget $B_{tr}$; ridge 
    $\lambda$.\\
    \textbf{Output:} Selected set $\mathcal{S}'$, weights 
    $\mathbf{w}$.
\end{flushleft}
\begin{algorithmic}[1]
    \STATE Initialize selected set $\mathcal{S}' \leftarrow \emptyset$, 
           weights $\mathbf{w} \leftarrow \emptyset$.
    \FOR{$k = 1$ to $B_{tr}$}
        \STATE Initialize best candidate $u^*$, 
               min error $\varepsilon^* \leftarrow \infty$.
        \FOR{each candidate $u \in \mathcal{S}^{(t)} \setminus \mathcal{S}'$}
            \STATE Form trial set $\mathcal{S}_{\mathrm{trial}} 
                   \leftarrow \mathcal{S}' \cup \{u\}$.
            \STATE Solve weights: 
                   $\mathbf{w}^*_{\mathrm{trial}} \leftarrow 
                   (\mathbf{G}_{\mathcal{S}_{\mathrm{trial}}} 
                   + \lambda \mathbf{I})^{-1} 
                   \mathbf{b}_{\mathcal{S}_{\mathrm{trial}}}$.
            \STATE Compute objective error: 
                   $\mathcal{L} \leftarrow 
                   \| \nabla\tilde\ell_{\mathrm{val}} 
                   - \sum_{m \in \mathcal{S}_{\mathrm{trial}}} 
                     w^*_m \nabla\ell_m \|^2 
                   + \lambda \|\mathbf{w}^*_{\mathrm{trial}}\|^2$.
            \IF{$\mathcal{L} < \varepsilon^*$}
                \STATE $u^* \leftarrow u$, 
                       $\varepsilon^* \leftarrow \mathcal{L}$, 
                       $\mathbf{w} \leftarrow \mathbf{w}^*_{\mathrm{trial}}$.
            \ENDIF
        \ENDFOR
        \STATE $\mathcal{S}' \leftarrow \mathcal{S}' \cup \{u^*\}$.
    \ENDFOR
    \STATE \textbf{return} $\mathcal{S}', \mathbf{w}$
\end{algorithmic}
\end{algorithm}

This appendix details the standard Orthogonal Matching Pursuit 
(OMP) sub-routine and its relation to the simpler greedy filter 
used in our main Algorithm~\ref{alg: ft w. ds}.

\textbf{OMP and the precomputed objective.}
OMP is a greedy algorithm that iteratively selects samples to 
minimize the residual error between the current weighted 
combination and the target vector. To make iterative selection 
tractable, we expand the squared $L_2$ objective in 
Eq.~\ref{eq:matching}:
\begin{equation}
\label{eq:expansion}
\mathcal{L}(\mathbf{w}, \mathcal{S}) 
= \| \nabla\tilde\ell_{\mathrm{val}} \|_2^2 
- 2 \sum_{i \in \mathcal{S}} w_i\, b_i 
+ \sum_{i \in \mathcal{S}} \sum_{j \in \mathcal{S}} 
  w_i w_j\, G_{ij} 
+ \lambda \|\mathbf{w}\|_2^2,
\end{equation}
where $b_i$ and $G_{ij}$ are the alignment-vector and Gram-matrix 
entries defined in Section~\ref{sec:method:ftw}. Both 
$\mathbf{b} \in \mathbb{R}^{|\mathcal{S}^{(t)}|}$ and 
$\mathbf{G} \in \mathbb{R}^{|\mathcal{S}^{(t)}| \times 
|\mathcal{S}^{(t)}|}$ are computed once per step on the projected 
low-rank gradients (Section~\ref{sec:efficient_gradient}); 
preconditioning is absorbed entirely into $\mathbf{b}$ via the 
target-side rescaling in Eq.~\ref{eq:adjoint_transfer}.

\textbf{The OMP sub-routine.}
Algorithm~\ref{alg: omp sub} details the selection process. At 
each iteration, we search for a candidate $u$ that, when added to 
the current set $\mathcal{S}'$, yields the minimal objective 
value. For a fixed trial set $\mathcal{S}_{\mathrm{trial}} = 
\mathcal{S}' \cup \{u\}$, the optimal weights have a closed-form 
ridge solution:
\begin{equation}
\mathbf{w}^*_{\mathrm{trial}} 
= (\mathbf{G}_{\mathcal{S}_{\mathrm{trial}}} + \lambda \mathbf{I})^{-1} 
  \mathbf{b}_{\mathcal{S}_{\mathrm{trial}}}.
\end{equation}
Since the target batch size $B_{tr}$ is small in the online 
setting, solving this system is negligible relative to gradient 
computation.

\textbf{Why our main algorithm uses a simplified variant.}
In the main Filter-then-Weight procedure 
(Algorithm~\ref{alg: ft w. ds}, Stage 1), we adopt a simplified 
version of OMP that skips the inner ridge solve when each new 
sample is added, retaining only the residual update 
$r \leftarrow r - \nabla\ell_{i^*}$. Solving for optimal weights 
at every greedy step, as standard OMP does, attempts to construct 
a linear combination that perfectly fits the noisy validation 
gradient—often selecting samples that cancel out random 
fluctuations in the validation batch rather than capturing the 
general descent direction. Our two-stage decoupling instead 
separates directional agreement (Stage 1, unit weights) from 
precise weight assignment (Stage 2, NNLS over the fixed set), 
reducing the noise propagated through greedy selection.

\subsection{Pseudo-codes of Filter-then-Weight}
\label{section: our pseudocodes}

Algorithm~\ref{alg: ft w. ds} summarizes one iteration of 
Filter-then-Weight. Each step proceeds in four stages: 
gradient acquisition prepares projected per-sample factors for both 
the candidate and validation pools (Stage~0); preconditioning and 
greedy filtering identify a geometrically diverse subset 
$\mathcal{S}'$ under unit weights, with the residual update 
$r \leftarrow r - \nabla\ell_{i^*}$ enforcing diminishing returns 
across selected samples (Stage~1); NNLS reweighting then refines 
the coefficients $\mathbf{w}^*$ over the fixed set 
$\mathcal{S}'$ to form a precise composite gradient (Stage~2); 
and the optimizer applies its update rule $\mathcal{P}_t$ to this 
composite gradient (Stage~3). The two-stage structure ensures 
filtering and weighting each operate at the noise level they can 
handle, as discussed in Section~\ref{sec:method:ftw}.

\begin{algorithm}[!t]
\small 
\caption{\texttt{Filter-then-Weight at Step $t$}}
\label{alg: ft w. ds}
\begin{algorithmic}[1]
    \REQUIRE Model $\theta_{t-1}$, Adam state $\hat{v}_{t-1}$; 
             candidate pool $\mathcal{S}^{(t)}$, validation pool 
             $\mathcal{S}_{\mathrm{val}}^{(t)}$; budget 
             $B_{tr}$, ridge $\lambda$, learning rate $\eta_t$.
    \ENSURE Updated model $\theta_t$.
    \STATE \textbf{Stage 0: Gradient Acquisition.} 
    \STATE Run forward/backward on $\mathcal{S}^{(t)}$ and 
           $\mathcal{S}_{\mathrm{val}}^{(t)}$.
    \STATE Store projected activations and gradients 
           $\{a_i, g_i\}_{i \in \mathcal{S}^{(t)}}$, 
           $\{a_j, g_j\}_{j \in \mathcal{S}_{\mathrm{val}}^{(t)}}$.
    
    \STATE \textbf{Stage 1: Preconditioning and Greedy Filtering.} 
    \STATE Compute averaged validation gradient 
           $\nabla\ell_{\mathrm{val}}$.
    \STATE Compute preconditioned target 
           $\nabla\tilde\ell_{\mathrm{val}} \leftarrow 
           \mathbf{D}_{t-1} \odot \nabla\ell_{\mathrm{val}}$ 
           using $\hat{v}_{t-1}$.
    \STATE Initialize $\mathcal{S}' \leftarrow \emptyset$, 
           residual $r \leftarrow \nabla\tilde\ell_{\mathrm{val}}$.
    \WHILE{$|\mathcal{S}'| < B_{tr}$}
        \STATE $i^* \leftarrow \arg\max_{i \notin \mathcal{S}'} 
               \langle \nabla\ell_i, r \rangle$.
        \STATE $\mathcal{S}' \leftarrow \mathcal{S}' \cup \{i^*\}$.
        \STATE $r \leftarrow r - \nabla\ell_{i^*}$.
    \ENDWHILE
    
    \STATE \textbf{Stage 2: NNLS Reweighting.} 
    \STATE Form Gram matrix $\mathbf{G}_{\mathcal{S}'}$ and 
           alignment vector $\mathbf{b}_{\mathcal{S}'}$ on 
           $\mathcal{S}'$.
    \STATE Solve $\mathbf{w}^*$ in Eq.~\ref{eq:matching} 
           via NNLS.
    
    \STATE \textbf{Stage 3: Model Update.} 
    \STATE Composite gradient: 
           $\bar\nabla \ell^{(t)} \leftarrow 
           \sum_{i \in \mathcal{S}'} w_i^* \nabla\ell_i$.
    \STATE $\theta_t \leftarrow \theta_{t-1} - \eta_t\, 
           \mathcal{P}_t(\bar\nabla \ell^{(t)})$.
\end{algorithmic}
\end{algorithm}

\section{Proofs}
\label{sec: proofs}
\subsection{Shapley value Using First-order Taylor Expansion}

We First discuss the first-order Taylor expansion as an approximation to the utility change after selecting and reweighting the training samples, such that $ U(S^{(t)} ; x_j) = l( \tilde{\theta_t}(S^{(t)} ), x_j) - l( \theta_{t-1}, x_j) \approx -\eta_t \nabla l_j (\tilde{\theta_t}(S^{(t)} ) - \theta_{t-1})$.

Importantly, our two-stage algorithm \textbf{separates the weighting stage from the Shapley value's subset selection}. This means a sample's weight is fixed and does not change depending on which subset it appears in. Because weights are assigned only after the batch is selected, they are independent of the probability of inclusion. The weights therefore only affect the scale of the gradient update, not the sampling process itself.

Using the definition for optimizer and weighting in Equation \ref{eq: taylor}, we derive 

\begin{equation}
\label{eq: optimizer_n_weight update}
U(S^{(t)} ; x_j) = -\eta_t \nabla l_j \mathcal{P}_t ( g (w) ) = -\eta_t \nabla l_j \mathcal{P}_t ( \sum_{x_i \in S^{(t)} } w_i \nabla l_i).
\end{equation}

Recall the Shapley value definition:

\begin{equation}
    \label{eq: shapley}
    \begin{split}
    \phi_i( U^{(t)}(x_i ;x_j) ) = \sum_{s=1}^{|S^{(t)}|} {|S^{(t)}|-1 \brack s -1}^{-1} \sum_{S \subseteq S^{(t)} \backslash x_i, |S| = s-1} [U^{(t)}(S \cup x_i ; x_j) - U^{(t)}(S ; x_j)].
    \end{split}
\end{equation}

Substituting Equation \ref{eq: optimizer_n_weight update} into the definition, and consider linearized $\mathcal{P}_t$, we get 

\begin{equation}
    \label{eq: shapley_first_order}
    \begin{split}
    \phi_i( U^{(t)}(x_i ;x_j) ) & = \nabla l_j \sum_{s=1}^{|S^{(t)}|} {|S^{(t)}|-1 \brack s -1}^{-1} \sum_{S \subseteq S^{(t)} \backslash x_i, |S| = s-1} [  \mathcal{P}_t ( \sum_{x_i \in S \cup x_i } w_i \nabla l_i) -  \mathcal{P}_t ( \sum_{x_i \in S } w_i \nabla l_i)] \\
    & = -\eta_t \nabla l_j \mathcal{P}_t (w_i \nabla l_i).
    \end{split}
\end{equation}

For 0-1 data selection problem, the utility of selecting sample $x_j$, $\phi_i( U^{(t)}(x_i ;x_j) ) = \nabla l_j \mathcal{P}_t (\nabla l_i)$; for weighting problem, $\phi_i( U^{(t)}(x_i ;x_j) ) = -\eta_t w_i \nabla l_j \mathcal{P}_t (\nabla l_i)$.

\subsection{Shapley value Using Second-order Taylor Expansion}
\label{[proof]:shapley}

Next, we discuss the second-order Taylor expansion at $\theta_{t-1}$ as an approximation, $ U(S^{(t)} ; x_j) = l( \tilde{\theta_t}(S^{(t)} ), x_j) - l( \theta_{t-1}, x_j) \approx \nabla l (\theta_{t-1}, x_j) (\tilde{\theta_t}(S^{(t)} ) - \theta_{t-1}) + \frac{1}{2} ( \tilde{\theta_t}(S^{(t)}) - \theta_{t-1})^T H^{(t)}_j ( \tilde{\theta_t}(S^{(t)}) - \theta_{t-1})$, where the Hessian matrix $H^{(t)}_j := \nabla^2 l_j$. Notably, the utility of the total set is not a linear sum of all individual samples, due to the presence of second-order Hessian matrix. In the rest of this paper, we substitute $H^{(t)}_j$ for $H_j$ for simplicity.

Similarly, updating with Equation \ref{eq: optimizer_n_weight update}, we draw connection between $\tilde{\theta_t}(S^{(t)})$ and $\theta_{t-1}$, i.e., $\tilde{\theta_t}(S^{(t)}) - \theta_{t-1} = - \eta_t \mathcal{P}_t( \sum_{z_i \in S^{(t)} } w_i \nabla l_i )$. Substituting into the original utility function, for training sample $x_i \in S^{(t)}$ and validation sample $x_j \in D_{val}$, we get candidate set utility 
\begin{equation}
\label{eq: 2-order Taylor}
    \begin{split}
        U(S^{(t)} ; x_j)  &= - \eta_t \nabla l_j \mathcal{P}_t( \sum_{z_i \in S^{(t)} } w_i \nabla l_i ) 
         + \frac{ \eta_t^2 }{2} \mathcal{P}_t( \sum_{z_i \in S^{(t)} } w_i \nabla l_i ) H_j \mathcal{P}_t( \sum_{z_k \in S^{(t)} } w_k \nabla l_k ), \\&\text{ where } \space H_j = \nabla^2 l_j.
    \end{split}
\end{equation}
We take the last item as an approximation of candidate set utility, denoted as $\hat{U}^{(t)}(S^{(t)} ; x_j)$. This expression can be derived from only the step size $\eta_t$ and the individual gradients of current model $\theta_{t-1}$ to involved training and validation samples, i.e., $\nabla l_i$ and $\nabla l_j$. 

% Though skipping requirement for actual model parameters $\tilde{\theta_t}(S^{(t)})$, the computation cost is barely reduced, which mainly lies in the gradient calculation for every possible combinations. Instead of computing all combinations, we use game-theoretic notion to fairly allocate utility of set to each sample, even when the training samples are jointly dependent.
% Following \cite{wang2024data}, $\phi_i( U(\cdot ;x_j) )$, the individual contribution of sample $x_i$ at step $t$ is computed with classic Shapley value~\citep{shapley1953value}, where utility is evaluated with validation point $x_j$. When the actual training batch $S^{(t)}$ is fixed, we can use closed-form formula to calculate the value of each point. 

Recall that $S^{(t)}$ is the final training data batch selected at step $t$.
For a sample $x_i$, denote $S \subseteq S^{(t)} \backslash x_i$ as any combinations without $x_i$.
Similarly, we consider only linearized $\mathcal{P}_t$.
By Equation \ref{eq: 2-order Taylor}, the marginal contribution of adding a new sample into the current candidate batch $S$ can be derived as

\begin{equation}
    \label{eq: margin}
    \begin{split}
    &U^{(t)}(S \cup x_i ; x_j) - U^{(t)}(S ; x_j) \\
    =& - \eta_t  \Bigl( \nabla l_j \mathcal{P}_t( \sum_{z_i \in S  \cup x_i} w_i\nabla l_i )  -  \nabla l_j \mathcal{P}_t( \sum_{z_i \in S } w_i \nabla l_i ) \Bigr)  \\
         &+ \frac{ \eta_t^2 }{2} \Bigl( \mathcal{P}_t( \sum_{z_i \in  S  \cup x_i } w_i\nabla l_i ) H_j \mathcal{P}_t( \sum_{z_k \in   S  \cup x_i } w_k \nabla l_k ) - \mathcal{P}_t( \sum_{z_i \in  S  } w_i \nabla l_i ) H_j \mathcal{P}_t( \sum_{z_k \in   S  } w_k \nabla l_k ) \Bigl)\\
    = & -\eta_t \mathcal{P}_t( w_i \nabla l_i) \nabla l_j + \frac{\eta^2_t}{2} \Bigl( 2 \mathcal{P}_t( w_i \nabla l_i) H_j   \mathcal{P}_t ( \sum_{x_k \in S} w_k \nabla l_k) + \mathcal{P}_t (w_i \nabla l_i) H_j \mathcal{P}_t (w_i \nabla l_i) \Bigr).
    \end{split}
\end{equation}

For the formula of Shapley value in Equation \ref{eq: shapley}, only the second item in Equation \ref{eq: margin} is variant cross different $S$. Substituting it into Equation $\ref{eq: shapley}$, we get $ \eta^2_t \mathcal{P}_t(w_i \nabla l_i) H^{(t)}_j \mathcal{P}_t ( \sum_{x_k \in S} w_k \nabla l_k) $. Based on the linearity of $\mathcal{P}_t$ and the symmetry property of Shapley values, the expectation of the sum term over all subsets $S$ is half the sum of the remaining samples: 

$$\mathbb{E}_S \left[ \sum_{x_k \in S} \nabla w_k l_k \right] = \frac{1}{2} \sum_{x_k \in S^{(t)} \setminus x_i} w_k \nabla l_k.$$

Substituting this back into Equation \ref{eq: shapley} and combining it with the invariant terms from Equation \ref{eq: margin}:

\begin{equation}
\begin{split}
\phi_i( U^{(t)}(x_i ;x_j) ) 
&= -\eta_t \mathcal{P}_t( w_i\nabla l_i) \nabla l_j + \frac{\eta^2_t}{2} \mathcal{P}_t (w_i \nabla l_i) H_j \mathcal{P}_t (w_i \nabla l_i) \\
&+ \eta^2_t \mathcal{P}_t(w_i \nabla l_i) H_j \mathcal{P}_t \left( \frac{1}{2} \sum_{x_k \in S^{(t)} \setminus x_i} w_k \nabla l_k \right) \\
&= -\eta_t \mathcal{P}_t( w_i \nabla l_i) \nabla l_j + \frac{\eta^2_t}{2} \mathcal{P}_t(w_i \nabla l_i) H_j \mathcal{P}_t \left( w_i \nabla l_i + \sum_{x_k \in S^{(t)} \setminus x_i} w_k \nabla l_k \right).
\end{split}
\end{equation}

Since $\nabla w_i l_i + \sum_{x_k \in S^{(t)} \setminus x_i} w_k \nabla l_k$ equals the sum over the entire batch $S^{(t)}$, we obtain the final expression:
\begin{equation}
\label{eq: final_shapley}
\phi_i( U^{(t)}(x_i ;x_j) ) = -\eta_t \mathcal{P}_t( w_i\nabla l_i) \nabla l_j + \frac{\eta^2_t}{2} \mathcal{P}_t( w_i \nabla l_i) H_j \mathcal{P}_t \left( \sum_{x_k \in S^{(t)}} w_k \nabla l_k \right).
\end{equation}

\subsection{Connection between Second-Order Contribution and Gradient Matching}
\label{appendix: connection}

Notably, even using approximation $H^{(t)}_j := \nabla^2 l_j$ as a parameter size matrix, the cost can be high for online computation. Therefore, previous work~\cite{wang2024greats} has simplified it further to an identity matrix $I$. Considering a validation batch $S_{val}$, utility of training sample combination $S^{(t)}$ using the second-order Taylor Expansion can be simplified as

\begin{equation}
\label{eq: simp_utility}
U(S^{(t)} ; x_j)  = - \eta_t \left( \mathcal{P}_t( \sum_{z_i \in S^{(t)} } w_i \nabla l_i ) \sum_{x_j \in S_{val}}\nabla l_j
         - \frac{ \eta_t }{2} \mathcal{P}_t( \sum_{z_i \in S^{(t)} } w_i \nabla l_i ) ^2 \right)
\end{equation}

Then, minimizing the total utility is equivalent to:

\begin{equation}
\label{eq:simp_utility_objective}
\min_{w}\; U(S^{(t)} ; x_j)  \equiv  \min_{w}\; -\mathcal{P}_t( \sum_{z_i \in S^{(t)} } w_i \nabla l_i ) \sum_{x_j \in S_{val}}\nabla l_j
         + \frac{ \eta_t }{2} \mathcal{P}_t( \sum_{z_i \in S^{(t)} } w_i \nabla l_i ) ^2 
\end{equation}

Meanwhile, expanding the gradient matching goal in Equation \ref{eq:matching}, we are optimizing over 

\begin{equation}
\label{eq:expanded_matching_objective}
\begin{split}
&\min_{w}\; 
\left\|
 \sum_{x_j \in S_{val}}\nabla l_j - \mathcal{P}_t
(\sum_{x_i\in S^{(t)}} w_i \nabla l_i)
\right\|_2^2
+\lambda\|w\|_2^2 \\
\equiv & \min_{w} - \mathcal{P}_t(\sum_{x_i\in S^{(t)}} w_i \nabla l_i) \sum_{x_j \in S_{val}}\nabla l_j 
+ \frac{1}{2} \mathcal{P}_t(
\sum_{x_i\in S^{(t)}} w_i \nabla l_i
) ^2 + \frac{1}{2} \lambda\|w\|_2^2.
\end{split}
\end{equation}

% Comparing Equation \ref{eq:simp_utility_objective} and \ref{eq:expanded_matching_objective}, we observe that the gradient matching objective shares the same first training-validation interaction item. Differently, the similarity between training sample is more punished, and an regularization item is added to prevent exploding weights.

Comparing Equation \ref{eq:simp_utility_objective} and 
\ref{eq:expanded_matching_objective}, both objectives share the 
same first-order alignment term 
$-\mathcal{P}_t(\sum_i w_i \nabla l_i) \sum_{x_j \in S_{val}} \nabla l_j$, 
which captures the training–validation interaction. The two 
objectives differ in three ways:

\begin{itemize}
\item \textbf{Quadratic term coefficient.} The second-order Taylor 
expansion weights $\mathcal{P}_t(\sum_i w_i \nabla l_i)^2$ by 
$\eta_t/2$, while gradient matching weights it by $1/2$. Since 
$\eta_t \sim 10^{-4}$ in fine-tuning, gradient matching amplifies 
this quadratic penalty by roughly four orders of magnitude. Rather 
than a flaw, this amplification acts as an \emph{implicit 
regularizer} on the update norm, discouraging the solver from 
constructing large composite gradients that may overfit a noisy 
validation direction.

\item \textbf{Pairwise interactions encoded in the quadratic term.} 
The quadratic term expands as $\mathcal{P}_t(\sum_i w_i \nabla l_i)^2 
= \sum_{i,j} w_i w_j \langle \mathcal{P}_t(\nabla l_i), 
\mathcal{P}_t(\nabla l_j) \rangle$, introducing pairwise 
interactions across selected samples. Highly correlated samples 
incur larger cross terms $w_i w_j \langle \cdot, \cdot \rangle$, 
naturally encoding diminishing returns and redundancy control. 
This is the structural property that motivates the subset-level 
matching formulation in Section~\ref{sec:method:utility}: pairwise 
structure is precisely what individual scoring discards.

\item \textbf{Ridge term.} The $\frac{\lambda}{2}\|w\|_2^2$ term 
in gradient matching has no analog in the second-order expansion. 
It serves a different role from the quadratic term: rather than 
controlling the geometry of $\mathcal{P}_t(\sum_i w_i \nabla l_i)$, 
it controls the mixture coefficients $w$ themselves, preventing 
degenerate solutions where a few samples receive arbitrarily 
large weights.
\end{itemize}

\textbf{Caveat: the isotropic-Hessian assumption.}
The connection above relies on $H_j \approx I$, a strong 
simplification at LLM scale where Hessian eigenspectra are highly 
anisotropic. A more accurate approximation such as 
K-FAC~\citep{martens2015optimizing} would replace 
$\mathcal{P}_t(\sum_i w_i \nabla l_i)^2$ with the Mahalanobis form 
$\mathcal{P}_t(\sum_i w_i \nabla l_i)^\top H_j 
\mathcal{P}_t(\sum_i w_i \nabla l_i)$, yielding pairwise terms 
$\langle \mathcal{P}_t(\nabla l_i), H_j \mathcal{P}_t(\nabla l_j) \rangle$ 
in the Gram matrix rather than the raw inner product. We derive 
the corresponding K-FAC-based scoring procedure and an efficient 
factorized computation in 
Appendix~\ref{[proof]: eq of 2 order}; in our preliminary 
experiments, however, the K-FAC variant did not yield improvements 
over the isotropic version on our benchmarks, and we leave a 
deeper investigation of richer Hessian approximations to future 
work.

\subsection{Inner-product and cosine maximization as squared matching}
\label{sec:inner_cosine}

% \sm{can you check the above paragraph? If you are happy with it, maybe we can remove the other text below and only keep the theorem - we can also send the theorem to A.2.4}
% \fx{I think its good to keep the current version and move the following rest to appendix. Also I will work on 3.5 later. PLZ go head and change 3.3 and 3.4}

The gradient matching design is a natural extension from the optimization goal defined in Equation \ref{eq: weighted objective}.
Formally, Theorem~\ref{thm:alignment_l2_unified} establishes that maximizing inner-product or cosine alignment is equivalent to minimizing a squared $\ell_2$ distance up to terms that depend only on the norm of the update.
In particular, when the optimizer-induced mapping $\mathcal{P}_t$ is linear (e.g., SGD) and the update norm is fixed or regularized, the two formulations admit the same set of optimal solutions.
Therefore, under these conditions, distance-based matching preserves the optimal solutions of the original alignment objective.

%When the optimizer-induced mapping $\mathcal{P}_t$
% is linear and the update norm is fixed or regularized, maximizing alignment and minimizing squared distance admit the same set of optimal solutions (Theorem 3.1). Full proofs and cosine variants are deferred to the appendix.

When $\mathcal{P}_t$ is nonlinear and optimizer-dependent (e.g., Adam/AdamW), strict equivalence between the two objectives no longer holds.
Nevertheless, minimizing the distance in the optimizer-induced feature space can be viewed as matching the effective update directions seen by the optimizer.
% This yields an optimizer-aware notion of subset utility that generalizes alignment maximization while retaining a meaningful geometric interpretation, making greedy selection a principled approximation in practice.

\begin{theorem}[Inner-product and cosine maximization as squared $\ell_2$ matching]
\label{thm:alignment_l2_unified}
Let $v \in \mathbb{R}^d$ be fixed, and let $h:\mathcal{W}\to\mathbb{R}^d$ be any mapping from a feasible set $\mathcal{W}$ to $\mathbb{R}^d$.

\paragraph{(i) Inner-product alignment.}
For any $w\in\mathcal{W}$,
\begin{equation}
\label{eq:ip_equiv}
\arg\max_{w\in\mathcal{W}} \ \langle v,h(w)\rangle
\;=\;
\arg\min_{w\in\mathcal{W}} \ \Big( \|v-h(w)\|_2^2 - \|h(w)\|_2^2 \Big).
\end{equation}

\paragraph{(ii) Cosine alignment.}
Assume $\|v\|_2 \neq 0$ and $\|h(w)\|_2 \neq 0$ for all $w\in\mathcal{W}$.
Define the cosine similarity objective
$C(w) := \frac{\langle v,\, h(w)\rangle}{\|v\|_2\,\|h(w)\|_2}$,
and the normalized directions $\hat v := \frac{v}{\|v\|_2}$, $\hat h(w) := \frac{h(w)}{\|h(w)\|_2}$.
Then
\begin{equation}
\label{eq:cos_equiv}
\arg\max_{w\in\mathcal{W}} \ C(w)
\;=\;
\arg\min_{w\in\mathcal{W}} \ \|\hat v - \hat h(w)\|_2^2.
\end{equation}
\end{theorem}

\label{[proof]: inner product/ cosine & l2 matching}

\begin{proof}
We prove the two statements separately.

\paragraph{(i) Inner-product case.}
By the polarization identity, for any $w\in\mathcal{W}$,
\begin{equation}
\label{eq:l2_expand_v}
\|v-h(w)\|_2^2
= \|v\|_2^2 + \|h(w)\|_2^2 - 2\langle v,h(w)\rangle.
\end{equation}
Rearranging~\eqref{eq:l2_expand_v} yields
\[
\langle v,h(w)\rangle
=
\frac{1}{2}\|v\|_2^2
+
\frac{1}{2}\big(\|h(w)\|_2^2 - \|v-h(w)\|_2^2\big).
\]
Since $\|v\|_2^2$ is constant with respect to $w$, maximizing
$\langle v,h(w)\rangle$ over $\mathcal{W}$ is equivalent (in the sense of sharing
the same set of optimizers) to minimizing
\[
\|v-h(w)\|_2^2 - \|h(w)\|_2^2,
\]
which proves~\eqref{eq:ip_equiv}.

\paragraph{(ii) Cosine case.}
We do not assume $\|v\|_2=1$ nor $\|h(w)\|_2=1$. Instead, we \emph{define}
$\hat v := v/\|v\|_2$ and $\hat h(w) := h(w)/\|h(w)\|_2$, which are well-defined
under the stated assumptions and satisfy
$\|\hat v\|_2=\|\hat h(w)\|_2=1$.

Expanding the squared distance gives
\begin{align}
\|\hat v - \hat h(w)\|_2^2
&= \|\hat v\|_2^2 + \|\hat h(w)\|_2^2 - 2\langle \hat v, \hat h(w)\rangle \nonumber\\
&= 2 - 2\left\langle \frac{v}{\|v\|_2},\, \frac{h(w)}{\|h(w)\|_2}\right\rangle \nonumber\\
&= 2 - 2\,\frac{\langle v,h(w)\rangle}{\|v\|_2\,\|h(w)\|_2} \nonumber\\
&= 2 - 2C(w).
\label{eq:cos_l2_relation}
\end{align}
The mapping $x \mapsto 2-2x$ is strictly decreasing. Hence minimizing
$\|\hat v - \hat h(w)\|_2^2$ over $w$ is equivalent to maximizing $C(w)$ over $w$,
which establishes~\eqref{eq:cos_equiv}.
\end{proof}

\subsection{Equivalence of dot product calculation order}
\label{[proof]: eq of 1 order}

An important property of Kronecker product is that, for matrices $A,B,C,D$ with matching dimensions,
\begin{equation}
    \label{eq: Kronecker product}
    < A\otimes B , C\otimes D>_F = < A, C>_F \otimes < B, D>_F.
\end{equation}
For dot product (Frobenius inner product) between $A$ and $B$, $< A, B>_F$, it also can be rewritten in trace of matrix:
\begin{equation}
    \label{eq: dot product}
     < A, B>_F  = < B, A>_F = tr(A^{\top} B)  = tr( B A^{\top}). 
\end{equation}
The dot product between vector $a$ and $b$ can also be denoted through transposition: $a \cdot b = a^Tb$.
These properties serve as an important foundation of proofs in this paper.

\textbf{Non-sequential data}. We start by proving the equivalence for non-sequential data:

\begin{equation}
\label{eq: alternative dot product}
    \sum_{x_j \in S_{val}^{(t)}} \nabla l_j \cdot \nabla l_i = (\sum_{x_j \in S_{val}^{(t)}} g_j^{\top}g_i)(a_j^{\top}a_i) = a_i^{\top} (\sum_{x_j \in S_{val}^{(t)}} a_j g_j^{\top}) g_i.
\end{equation}

Here, the derivative of loss to $i$-th pre-activation as $g_i = \frac{\partial l}{\partial s_i} \in \mathbb{R}^{d_2 \times 1}$, and the $i$-th input vector as $a_i \in \mathbb{R}^{ d_1 \times 1}$.
The first equation is derived by:
\begin{equation*}
     \nabla l_j \cdot \nabla l_i =
     (g_j  \otimes a_j)  \cdot (g_i  \otimes a_i) = (g_j \cdot g_i) (a_j \cdot a_i) = (g_j^{\top}g_i)(a_j^{\top}a_i),
\end{equation*}
where the first equation is by definition and the second and third one applies the property of Kronecker product in Equation \ref{eq: Kronecker product} and dot product, respectively.

Looking into the last item, we have

\begin{equation*}
\begin{split}
    a_i^{\top} (\sum_{x_j \in S_{val}^{(t)}} a_j g_j^{\top}) g_i
    & = \sum_{x_j \in S_{val}^{(t)}} a_i^{\top} (a_j g_j^{\top}) g_i \\
    & = \sum_{x_j \in S_{val}^{(t)}} (a_i^{\top} a_j) (g_j^{\top} g_i) \\
    & = \sum_{x_j \in S_{val}^{(t)}} (g_j^{\top}g_i)(a_j^{\top}a_i).
\end{split}
\end{equation*}

The first two equations rearranges the combination of vectors. The last item is swapping scalar $(g_j^{\top}g_i)$ and $(a_j^{\top}a_i)$. Proved for non-sequential data.\\

\textbf{Sequential data}. 

\begin{equation*}
    <\sum_{x_j \in S_{val}^{(t)}} \nabla l_j , \nabla l_i>_F = <\sum_{x_j \in S_{val}^{(t)}} g_j^{\top}g_i, a_j^{\top}a_i>_F = < g_i^{\top}, a_i^{\top} (\sum_{x_j \in S_{val}^{(t)}} a_j g_j^{\top})>_F.
\end{equation*}

In computation for sequential data, the 3-dimensional tensor $(B, T, d)$ is generally reshaped as $(B T, d)$ for better position-agnostic parallelization. We only need to substitute $a_i \in \mathbb{R}^{ d_1 \times 1}$ to $a_i \in \mathbb{R}^{ d_1 \times T}$ and $g_i = \frac{\partial l}{\partial s_i} \in \mathbb{R}^{d_2 \times 1}$ to $g_i = \frac{\partial l}{\partial s_i} \in \mathbb{R}^{d_2 \times T}$ in the above proving process. Notably, the equivalences remain, as all computation is inside the dot product. 
By denoting the trace matrix as $tr(\cdot)$, the second equation can be converted into 

\begin{equation*}
      tr((a_j^{\top}a_i)^{\top} (\sum_{x_j \in S_{val}^{(t)}} g_j^{\top}g_i ) ) = tr(  a_i^{\top} (\sum_{x_j \in S_{val}^{(t)}} a_j g_j^{\top})g_i),
\end{equation*}

which is essentially the same as Equation \ref{eq: alternative dot product}. Further, using the property of dot product, $< g_i^{\top}, a_i^{\top} (\sum_{x_j \in S_{val}^{(t)}} a_j g_j^{\top})>_F = < a_i, (\sum_{x_j \in S_{val}^{(t)}} a_j g_j^{\top})g_i>_F$. Using this equivalence, we can flexibly select the more efficient way for calculation according to dimensionality.

% The second-derivation, namely the ghost product item, can rewritten into:

% \begin{equation*}
% \begin{split}
%     S^{ghost}_i & = \sum_{i' = T(i-1)+1}^{Ti+1} \sum_{j' = 1}^{TB_{val}} (a_{j'}a_{i'}^{\top})(g_{j'}g_{i'}^{\top}) \\
%     & = \sum_{i' = T(i-1)+1}^{Ti+1} \sum_{{j'} = 1}^{TB_{val}} ( \textbf{A}_{tr}[{i'},:] \cdot  \textbf{A}_{val}[{j'},:]) \times (\textbf{B}_{tr}[{i'},:] \cdot  \textbf{B}_{val}[{j'},:]).
% \end{split}
% \end{equation*}

\subsection{Equivalence of vector-Hessian-vector calculation order}
\label{[proof]: eq of 2 order}

\textbf{Non-sequential data.}  
We first prove the corrected derivation. The goal is to compute the influence of a set of candidate samples on a target sample $i$, mediated by the Hessian estimated from a validation set.
\begin{equation*}
    \text{Score} = \nabla l_i^{\top} (\sum_{x_j \in S_{val}^{(t)}} H_j)(\sum_{x_k \in S_{cand}^{(c, t)} } \nabla l_k).
\end{equation*}

Assume the block-diagonal property for Hessian matrix in a layer level, and the K-FAC approximation where $H_{j,\ell} \approx G_{\ell} \otimes A_{\ell}$ represents the aggregate Hessian statistics derived from the validation batch. 
Using the property of the Kronecker product $(A \otimes B)(C \otimes D) = (AC) \otimes (BD)$, we expand the term for each layer $\ell$:

\begin{equation*}
\begin{split}
    & \nabla l_i^{\top} H_{val} \sum_{x_k \in S_{cand}^{(c, t)} } \nabla l_k \\
    = & \sum_{\ell = 1}^L ( g_{i,\ell}^{\top} \otimes {a}_{i,\ell}^{\top}) ( G_{\ell} \otimes A_{\ell}) \sum_{x_k \in S_{cand}^{(c, t)}}( g_{k,\ell} \otimes {a}_{k,\ell}) \\
    = & \sum_{\ell = 1}^L \sum_{x_k \in S_{cand}^{(c, t)}} \left[ ( g_{i,\ell}^{\top} \otimes {a}_{i,\ell}^{\top}) ( G_{\ell} \otimes A_{\ell}) ( g_{k,\ell} \otimes {a}_{k,\ell}) \right] \\
    = & \sum_{\ell = 1}^L \sum_{x_k \in S_{cand}^{(c, t)}} \left[ (g_{i,\ell}^{\top} G_{\ell} g_{k,\ell}) \otimes ({a}_{i,\ell}^{\top} A_{\ell} {a}_{k,\ell}) \right].
\end{split}
\end{equation*}

Since $g_{i,\ell}^{\top} G_{\ell} g_{k,\ell}$ and ${a}_{i,\ell}^{\top} A_{\ell} {a}_{k,\ell}$ are both scalars, the Kronecker product becomes a simple scalar multiplication. We now expand $G_{\ell} = \sum_{m=1}^{|B_{val}|} g_{m,\ell} g_{m,\ell}^{\top}$:

\begin{equation*}
\begin{split}
    g_{i,\ell}^{\top} G_{\ell} g_{k,\ell} 
    & = g_{i,\ell}^{\top} (\sum_{m=1}^{|B_{val}|} g_{m,\ell} g_{m,\ell}^{\top}) g_{k,\ell} \\
    & = \sum_{m=1}^{|B_{val}|} (g_{i,\ell}^{\top} g_{m,\ell}) (g_{m,\ell}^{\top} g_{k,\ell})  = \sum_{m=1}^{|B_{val}|} (g_{i,\ell} \cdot g_{m,\ell}) (g_{m,\ell} \cdot g_{k,\ell}).
\end{split}
\end{equation*}

Similarly, for the activation term:
\begin{equation*}
    {a}_{i,\ell}^{\top} A_{\ell} {a}_{k,\ell} = \sum_{m=1}^{|B_{val}|} (a_{i,\ell} \cdot a_{m,\ell}) (a_{m,\ell} \cdot a_{k,\ell}).
\end{equation*}

Substituting these back, we obtain the final corrected equation. Note that the summation over candidates $x_k$ must remain outside the product terms:

\begin{equation*}
    \text{Score} = \sum_{\ell = 1}^L \sum_{x_k \in S_{cand}^{(c, t)}} \left[ \left( \sum_{m \in S_{val}^{(t)}} ( g_{i,\ell} \cdot g_{m,\ell} ) ( g_{m,\ell} \cdot g_{k,\ell} ) \right) \cdot \left( \sum_{m \in S_{val}^{(t)}} ( a_{i,\ell} \cdot a_{m,\ell} ) ( a_{m,\ell} \cdot a_{k,\ell} ) \right) \right].
\end{equation*}

%Proved. \sm{maybe say, Hence it is proved.}

\textbf{Sequential data.} 
The derivation can be naturally extended to sequential data where activations and gradients are matrices (e.g., $a \in \mathbb{R}^{d_1 \times T}$). 
We replace the vector dot product with the Frobenius inner product $\langle X, Y \rangle_F = \text{Tr}(X^{\top} Y)$.
The logic of K-FAC remains consistent: it aggregates statistics across the temporal dimension, effectively treating time steps as additional batch samples.
Thus, the scalar terms simply change their operator:

\begin{equation*}
    \text{Score} = \sum_{\ell = 1}^L \sum_{x_k \in S_{cand}^{(c, t)}} \Bigg[  \left( \sum_{m \in S_{val}^{(t)}} \langle g_{i,\ell} , g_{m,\ell} \rangle_F \cdot \langle g_{m,\ell}, g_{k,\ell} \rangle_F \right) 
    \cdot  \left( \sum_{m \in S_{val}^{(t)}} \langle a_{i,\ell}, a_{m,\ell} \rangle_F \cdot \langle a_{m,\ell} , a_{k,\ell} \rangle_F \right) \Bigg].
\end{equation*}

\textbf{Empirical note.}
We implemented the full K-FAC-based scoring above as a drop-in 
replacement for the raw Gram matrix in 
Section~\ref{sec:method:ftw}, intended to test whether richer 
second-order information improves data selection. In our 
preliminary experiments on the same Llama-3.2-1B/TyDiQA setup as 
Section~\ref{sec:ablation}, however, the K-FAC variant produced 
no consistent improvement over the isotropic-Hessian (raw inner 
product) baseline, while incurring substantially higher per-step 
overhead from maintaining and applying $G_\ell, A_\ell$ at each 
layer. A possible explanation is that LoRA's low-rank structure 
already captures most of the curvature variance relevant for 
selection, making the additional K-FAC information largely 
redundant; a more rigorous investigation—including alternative 
Hessian approximations such as diagonal Fisher and partial-layer 
K-FAC—is left to future work. We include the full derivation here 
for completeness and to document the design space explored.

\section{Complexity Comparison}
\label{sec: Complexity Comparison}

In sequential data, consider pre-activation derivative $g_i, g_j \in \mathbb{R}^{T \times d_2}$ and input $a_i, a_j \in \mathbb{R}^{T \times d_1}$. For simplicity of notation, we ignore all layer indicators, $\ell$, when not required, and assume input and output dimension $d_1$ and $d_2$ are the maximum value for all layers. This is a valid assumption, as we use random projection with a constant maximum projection dimension for all layers in implementation.

\subsection{First-Order Derivative}

\textbf{Naive}. 

\begin{equation*}
    <\sum_{x_j \in S_{val}^{(t)}} \nabla l_j , \nabla l_i>_F = <\sum_{x_j \in S_{val}^{(t)}} a_j g_j^{\top}, a_i g_i^{\top}>_F
\end{equation*}

In this equation, $( (a_j^{\top} g_j)$ and $(a_i^{\top} g_i)$ takes $\mathcal{O}(T^2d_1d_2)$ in time and $\mathcal{O}(d_1d_2)$ to store the resulting matrix, respectively. The final dot product $(a_j^{\top} g_j)  \cdot (a_i^{\top} g_i)$ takes $O(d_1d_2)$ in time. Considering $\nabla l_j \cdot \nabla l_i$ needs to calculated for every pair of training and validation data in the batch, the total time complexity for all layers is $\mathcal{O}(LT^2B_{tr}B_{val}d_1d_2)$, space complexity is $\mathcal{O}(LB_{tr}B_{val}d_1d_2)$.\\\\

\textbf{Ghost}.

\begin{equation*}
<\sum_{x_j \in S_{val}^{(t)}} \nabla l_j , \nabla l_i>_F = <\sum_{x_j \in S_{val}^{(t)}} g_j^{\top}g_i, a_j^{\top}a_i>_F
\end{equation*}

Using ghost dot-product, $(g_j^{\top}g_i)$ takes $\mathcal{O}(T^2d_2)$ in time and generates a matrix requiring $\mathcal{O} (T^2)$ storage space. Similarly, $(a_j^{\top}a_i)$ takes $\mathcal{O}(T^2d_1)$ in time and $\mathcal{O} (T^2)$ in space. The final dot product is $\mathcal{O} (T^2)$ in time. Also, we need to consider every pair of training and validation data in the batch.
Nonetheless, as we only store all activations in the forward propagation ($\mathcal{O}(LT(B_{tr}+B_{val})d_1)$), and discard used activations during back-propagation, the required space is not linearly increasing with layer number. The peak space requirement is in backpropagation of last layer, i.e., $\mathcal{O}(LTd_1 + B_{tr}B_{val}T^2)$. Therefore, the total time complexity is $\mathcal{O} (LT^2B_{tr}B_{val} ( d_1 + d_2) )$, and space complexity is $\mathcal{O}(LT(B_{tr}+B_{val})d_1 + T^2B_{tr}B_{val})$.

\textbf{Our Method}.

\begin{equation}
<\sum_{x_j \in S_{val}^{(t)}} \nabla l_j , \nabla l_i>_F = < g_i^{\top}, a_i^{\top} (\sum_{x_j \in S_{val}^{(t)}} a_j g_j^{\top})>_F
\end{equation}

The summation over outer products of all validation samples $\sum_{x_j \in S_{val}^{(t)}} a_j g_j^{\top}$ takes $\mathcal{O} (B_{val}Td_1d_2)$ in time and $\mathcal{O}(d_1d_2)$ for storage space. Then the combination with the first $a_i^{\top}$ takes $\mathcal{O}(Td_1d_2)$ in time and $\mathcal{O}(Td_2)$ in space. Lastly, we make a dot product with $g_i$, consuming $\mathcal{O}(Td_2)$ in time and $\mathcal{O}(Td_2)$ for the final result. 
Similarly, all productions are calculated layer by layer. Therefore, the final total time complexity is $\mathcal{O}( LTB_{tr}d_2(B_{val}d_1 + T))$. The peak memory requirement is only $\mathcal{O}(LT(B_{tr}+B_{val})d_1 + B_{tr}d_2\max(d_1, T) )$. As discussed in Section \ref{[proof]: eq of 1 order}, symmetrically, we can choose to first combining with $g_i$ than $a_i^{\top}$, which is equivalent in math. This option takes $\mathcal{O}( LTB_{tr}d_1(B_{val}d_2 + T))$ in time, and $\mathcal{O}(LT(B_{tr}+B_{val})d_1 + B_{tr}d_1\max(d_2, T) )$ in space.

\section{Additional Experiment Details}

\subsection{Setting}
\label{subsec: setting}

All experiments were implemented using the PyTorch framework and 
the Hugging Face Transformers and PEFT libraries, running on NVIDIA 
A100 (40\,GB), H100 (80\,GB), H200 (141\,GB), and Tesla V100  (32\,GB)GPUs. 
We use Llama-3.2-1B as our primary backbone, fine-tuned via 
Low-Rank Adaptation (LoRA). Unless otherwise stated, the LoRA rank 
is $r=8$, scaling factor $\alpha=32$, and dropout rate $0$, applied 
to all linear layers (query, key, value, output, gate, up, and 
down projections). We optimize with Adam ($\beta_1{=}0.9$, 
$\beta_2{=}0.999$), maximum learning rate $1\times10^{-4}$ and 
minimum $1\times10^{-5}$, with linear warmup over the first 800 
steps followed by decay over 2{,}000 steps. The training batch 
size is 8 with maximum sequence length 2{,}048, using 
mixed-precision (BF16) via Hugging Face Accelerate.

\subsection{Datasets and Models}
\label{app:licenses}

We use the following publicly released datasets and pre-trained 
models, all consistent with their licenses for non-commercial 
research purposes:

\textbf{Datasets.}
\begin{itemize}
    \item \textbf{Open-Instruct}~\citep{lambert2024tulu3} 
    --- ODC-BY 1.0.
    \item \textbf{FLAN v2}~\citep{longpre2023flan} --- Apache 2.0.
    \item \textbf{Chain-of-Thought}~\citep{wei2022chain} 
    --- Apache 2.0.
    \item \textbf{Dolly}~\citep{conover2023dolly} --- CC-BY-SA 3.0.
    \item \textbf{OASST1}~\citep{kopf2023openassistant} --- 
    Apache 2.0.
    \item \textbf{MMLU}~\citep{hendrycksmeasuring} --- MIT License.
    \item \textbf{TyDiQA}~\citep{clark2020tydi} --- Apache 2.0.
\end{itemize}

\textbf{Models.}
\begin{itemize}
    \item \textbf{Llama-3.2-1B} --- Meta Llama 3.2 Community 
    License Agreement.
    \item \textbf{Qwen3-0.6B} --- Apache 2.0.
\end{itemize}

\subsection{Ablation Variant Definitions}
\label{app:ablation_variants}

We list the full set of variants studied in 
Section~\ref{sec:ablation}. Variants 5 and 6 are excluded from the 
multi-seed comparison in Table~\ref{tab:ablation_studies} due to 
training divergence; their preliminary single-seed trajectories are 
shown in Figure~\ref{fig:ablation_f1}.

\begin{enumerate}
    \item \textbf{Hard Filter + Reweight.} Replaces the greedy 
    filtering (Stage 1) with a top-$k$ selection based on 
    optimizer-aware scores, while retaining the NNLS weighting 
    stage (Stage 2).
    
    \item \textbf{Optimizer-Aware Filter Only.} Selects top 
    candidates using optimizer-aware scores without the subsequent 
    weighting stage (similar to LESS).
    
    \item \textbf{Vanilla Filter Only.} Selects candidates based 
    on raw gradient dot products (similar to TracIn), without 
    optimizer-aware preconditioning or the weighting stage.
    
    \item \textbf{Vanilla Filter + Reweight.} Implements the 
    two-stage Filter-then-Weight framework using raw gradients for 
    both filtering and weighting, removing all optimizer-aware 
    preconditioning.
    
    \item \textbf{Unbounded Weights.} Relaxes the non-negative 
    constraint in Eq.~\ref{eq:matching}, allowing weights 
    to take negative values.
    
    \item \textbf{Token-Level Score.} Restricts gradient similarity 
    computation to identical token positions, ignoring 
    sequence-level context.
    
    \item \textbf{Ours.} The full Filter-then-Weight algorithm proposed in this paper.
\end{enumerate}

\begin{table}[t]
    \centering
    \caption{\textbf{Running Time Comparison.} We report the average running time (seconds per step) on NVIDIA H200 GPUs.}
    \label{tab:running_time}
    \vspace{0.2cm}
    \begin{small}
    \begin{tabular}{lcc}
        \toprule
        \textbf{Method} & \textbf{Time (s / step)} & \textbf{Relative to SFT} \\
        \midrule
        \multicolumn{3}{l}{\textit{Standard Training}} \\
        Full Data (SFT, 4 steps) & 5.75 & 1.0$\times$ \\
        \midrule
        \multicolumn{3}{l}{\textit{Selection Baselines (Online Adapted)}} \\
        TracIn \citep{TracIn} & 4.12 & 0.72$\times$ \\
        LESS \citep{xia2024less} & 4.14 & 0.72$\times$ \\
        GREATS \citep{wang2024greats} & 6.23 & 1.08$\times$ \\
        GRAD-MATCH \citep{killamsetty2021grad} & 12.13 & 2.11$\times$ \\
        \midrule
        \multicolumn{3}{l}{\textit{Ablation: Implementation Efficiency}} \\
        Ours (w/ Ghost Dot-Product) & OOM  & OOM \\
        \textbf{Ours (Top-K filtering)} & \textbf{4.49} & \textbf{0.78$\times$} \\
         \textbf{Ours (Greedy filtering)} & \textbf{6.47} & \textbf{1.12$\times$} \\
        \bottomrule
    \end{tabular}
    \end{small}
\end{table}

\subsection{In-iteration Running Time Analysis}
\label{subsec: more results}

We provide a running comparison over baselines, with averaged running time per step in seconds. Out of efficiency concern, we adopt the derivative computation order introduced in Section \ref{sec:efficient_gradient}. We also compare over a variant of our method using ghost dot-product~\citep{wang2024data}, which is out of memory with projection dimension of 256. The results are given in Table \ref{tab:running_time}. Note that full training setting we train on every arriving mini-batch without gradient accumulation (set to be 4 in this paper), so it can be slower than selection methods.

As discussed in Section \ref{sec:ablation}, the top-k filtering based can achieve comparable peak performance, with little increased cost compared with filter-only methods. To achieve stable performance, it is suggested to use greedy filtering is necessary, which introduce extra computational cost for training gradient similarity in greedy selection stage. Nonetheless, it is still acceptable compared to the SFT baseline. Compared to GREATS, it takes a similar greedy strategy in the selection stage, but requires an additional weight solver. Fortunately, we can reuse computed gradient similarity score in the selection stage for subset samples, resulting in only slightly increased overhead.

\subsection{Empirical Memory Frontier at Long Context}
\label{sec:memory_frontier}

This appendix expands the complexity discussion in 
Section~\ref{sec:efficient_gradient}. While 
Table~\ref{tab:complexity_comparison} establishes our 
$\mathcal{O}(T B_{tr} d_2)$ score-tensor cost against Ghost's 
$\mathcal{O}(T^2 B_{tr} B_{val})$ asymptotically, whether this gap 
binds in practice depends on what else lives in GPU memory during a 
training step. The long-context feasibility frontier is shaped by 
two orthogonal memory axes—a \emph{model-side floor} 
(weights, activations, and notably the cross-entropy logits buffer 
of size $\mathcal{O}(B T V)$, where $V \approx 128$K for Llama-3) 
and an \emph{influence-side ceiling} (the score tensor of 
Eq.~\ref{eq:reordered inner-prod}). Each axis must be addressed 
separately: an asymptotic advantage on one becomes empirically 
observable only after the other has been brought below it.

\subsubsection{Influence-Side Microbenchmark}
\label{sec:mf_microbench}

We isolate influence-side scaling with a pure-tensor microbenchmark 
on a single 80\,GB H100. Synthetic gradient buffers of the exact 
pipeline shapes (activations $B_{tr} \cdot T \times d_1$ and 
backprops $B_{tr} \cdot T \times d_2$, analogously for validation 
tokens) are fed directly to the score-tensor kernel of 
Eq.~\ref{eq:reordered inner-prod} in both Ghost and our forms; the 
model forward and cross-entropy are bypassed entirely. Setting: 
$B_{val}{=}4$, $d_1{=}d_2{=}32$ (per-LoRA-layer rank), fp32 
intermediates, 1 warmup + 3 timed iterations. 
Tables~\ref{tab:mf_empirical_bottleneck} 
and~\ref{tab:mf_empirical_walltime} report peak memory and 
wallclock.

\begin{table}[h]
\centering
\small
\caption{Empirical peak GPU memory (MB) of the score tensor under 
the Llama-3.2-1B LoRA setting. \textbf{Bold} = OOM on H100 80\,GB.}
\label{tab:mf_empirical_bottleneck}
\setlength{\tabcolsep}{4pt}
\resizebox{0.9\textwidth}{!}{%
\begin{tabular}{c c rrrrrrr}
\toprule
$B_{tr}$ & Method & $T{=}512$ & $T{=}1024$ & $T{=}2048$ & $T{=}4096$ & $T{=}8192$ & $T{=}16384$ & $T{=}65536$ \\
\midrule
\multirow{2}{*}{8}   & Ghost & 136 & 439 & 1{,}651 & 6{,}489 & 25{,}829 & \textbf{OOM} & \textbf{OOM} \\
                     & Ours  &  36 &  39 &    44 &     55 &     76 &      118 &      371 \\
\midrule
\multirow{2}{*}{32}  & Ghost & 441 & 1{,}654 & 6{,}495 & 25{,}842 & \textbf{OOM} & \textbf{OOM} & \textbf{OOM} \\
                     & Ours  &  43 &   52 &    69 &    105 &    177 &      321 &     1{,}183 \\
\midrule
\multirow{2}{*}{64}  & Ghost & 848 & 3{,}273 & 12{,}955 & 51{,}646 & \textbf{OOM} & \textbf{OOM} & \textbf{OOM} \\
                     & Ours  &  51 &   68 &   103 &    173 &    312 &      591 &     2{,}265 \\
\midrule
\multirow{2}{*}{128} & Ghost & 1{,}662 & 6{,}511 & 25{,}874 & \textbf{OOM} & \textbf{OOM} & \textbf{OOM} & \textbf{OOM} \\
                     & Ours  &     68  &    102  &    171  &    308 &    583 &    1{,}132 &    4{,}429 \\
\bottomrule
\end{tabular}
}
\end{table}

\begin{table}[h]
\centering
\small
\caption{Empirical wallclock (ms, mean over 3 calls) of one 
score-tensor invocation. Same setting as 
Table~\ref{tab:mf_empirical_bottleneck}.}
\label{tab:mf_empirical_walltime}
\setlength{\tabcolsep}{4pt}
\resizebox{0.95\textwidth}{!}{%
\begin{tabular}{c c rrrrrrr}
\toprule
$B_{tr}$ & Method & $T{=}512$ & $T{=}1,024$ & $T{=}2,048$ & $T{=}4,096$ & $T{=}8,192$ & $T{=}16,384$ & $T{=}65,536$ \\
\midrule
\multirow{2}{*}{8}   & Ghost & 0.18 & 0.46 & 1.55 & 5.59 & 22.30 & \textbf{OOM} & \textbf{OOM} \\
                     & Ours  & 0.20 & 0.13 & 0.12 & 0.15 & 0.08  & 0.12         & 0.35         \\
\midrule
\multirow{2}{*}{64}  & Ghost & 0.75 & 2.87 & 11.29 & 48.73 & \textbf{OOM} & \textbf{OOM} & \textbf{OOM} \\
                     & Ours  & 0.08 & 0.09 &  0.12 &  0.20 & 0.32         & 0.60         & 2.05         \\
\midrule
\multirow{2}{*}{128} & Ghost & 1.48 & 5.68 & 22.46 & \textbf{OOM} & \textbf{OOM} & \textbf{OOM} & \textbf{OOM} \\
                     & Ours  & 0.08 & 0.11 & 0.17  & 0.30 & 0.57 & 1.03 & 3.97 \\
\bottomrule
\end{tabular}
}
\end{table}

The measurements validate the asymptotic prediction. Ghost's peak 
memory grows by exactly $4\times$ per $T$-doubling, 
matching the $T^2$ term, while ours grows by $\le 2\times$, matching 
the $T$ term. Ghost OOMs once $B_{tr} \cdot T^2$ exceeds roughly 
$2.6 \times 10^9$ tokens$^2$; ours has no observed OOM at any cell 
tested. Wallclock follows the same scaling: at equal feasibility 
($B_{tr}{=}8, T{=}8192$), our score-step is $278\times$ faster 
($22.3$\,ms vs.\ $0.08$\,ms).

\subsubsection{End-to-End Frontier with Model-Side Mitigation}
\label{sec:mf_e2e}

The microbenchmark above isolates the influence-side advantage, but 
end-to-end training is also bounded by the model-side floor. Under 
stock PyTorch, training in our reference setting (Llama-3.2-1B 
LoRA, $B_{tr}{=}8$, $B_{val}{=}4$) OOMs at $T{=}4096$, with the 
binding allocation being the $15.6$\,GB cross-entropy logits buffer 
rather than the score tensor. We therefore activate Liger 
Kernel~\citep{hsuliger}, which streams the lm\_head matmul 
through softmax and cross-entropy without materializing the 
$(B \cdot T, V)$ logits tensor, reducing the loss-step memory from 
$\mathcal{O}(B T V)$ to $\mathcal{O}(B T \cdot \mathrm{hidden})$ 
(an $\approx 60\times$ saving for Llama-3). 
Table~\ref{tab:mf_liger_oom_frontier} shows that Liger extends the 
end-to-end frontier from $T{=}4096$ to $T{=}16384$.

\begin{table}[h]
\centering
\small
\caption{End-to-end OOM frontier with and without Liger Kernel 
(Llama-3.2-1B LoRA, $B_{tr}{=}8$, $B_{val}{=}4$, single H100 80\,GB).}
\label{tab:mf_liger_oom_frontier}
\setlength{\tabcolsep}{10pt}
\begin{tabular}{c c c}
\toprule
$T$ & Stock PyTorch & + Liger \\
\midrule
$2{,}048$  & \checkmark & \checkmark \\
$4{,}096$  & \textbf{OOM} ($15.6$\,GB CE) & \checkmark \\
$8{,}192$  & --- & \checkmark \\
$16{,}384$ & --- & \checkmark \\
$32{,}768$ & --- & \textbf{OOM} \\
\bottomrule
\end{tabular}
\end{table}

At this post-Liger frontier, the asymptotic gap from 
Table~\ref{tab:complexity_comparison} becomes the binding 
constraint: at $T{=}16384, B_{tr}{=}8, B_{val}{=}4$, the Ghost 
score tensor would require $\approx 137$\,GB 
($16384^2 \cdot 32 \cdot 4$ entries, fp32), exceeding GPU capacity, 
while ours uses $118$\,MB. The two axes thus compose 
multiplicatively: model-side mitigation shrinks the method-agnostic 
floor (cross-entropy), and our influence-side decomposition shrinks 
the method-specific ceiling (score tensor). Notably, with better optimization, our method can scale up to even longer context. 

\newpage

\end{document}